\documentclass[conference]{IEEEtran}

\IEEEoverridecommandlockouts
\usepackage{cite}
\def\BibTeX{{\rm B\kern-.05em{\sc i\kern-.025em b}\kern-.08em
    T\kern-.1667em\lower.7ex\hbox{E}\kern-.125emX}}

\usepackage[ruled, vlined, linesnumbered]{algorithm2e}
\usepackage[utf8]{inputenc}
\usepackage[english]{babel}
\usepackage{mathptmx} 
\usepackage{amsmath} 
\usepackage{amssymb}
\usepackage{amsfonts}
\usepackage{amsthm}
\usepackage{mathtools}

\newtheorem{claim}{Claim}
\newtheorem{corollary}{Corollary}
\usepackage{bm}
\usepackage{cite}
\usepackage{comment}
\allowdisplaybreaks 

\usepackage{array}

\newtheorem{problem}{Problem}
\usepackage[pdftex]{graphicx} 
\usepackage{subcaption}
\usepackage{epstopdf}
\usepackage{hyperref} 
\usepackage{color}
\usepackage{xcolor}
\definecolor{Questions}{HTML}{1F77B4}

\newtheorem{theorem}{Theorem}
\numberwithin{theorem}{section}

\usepackage[capitalise]{cleveref}

\usepackage{cancel}

\usepackage{amssymb}

\usepackage{tikz}
\usetikzlibrary{calc}
\usetikzlibrary{arrows,positioning} 
\usetikzlibrary{intersections, patterns}
\usetikzlibrary{shapes}
\usepackage{mathtools}
\usepackage{graphicx}
\usepackage[cmtip,all]{xy}
\newcommand{\longsquiggly}{\xymatrix{{}\ar@{~>}[r]&{}}}

\begin{document}
\title{Active Collaborative Localization in \\ Heterogeneous Robot Teams}


\author{Igor Spasojevic$^{*}$, Xu Liu$^{*}$, Alejandro Ribeiro$^{}$, George J. Pappas$^{}$, Vijay Kumar$^{}$
\thanks{
*Equal contribution. 
This work was supported by The Institute for Learning-Enabled Optimization at Scale (TILOS) funded by the National Science Foundation (NSF) under NSF Grant CCR-2112665, IoT4Ag ERC funded through NSF Grant EEC-1941529, the ARL DCIST CRA W911NF-17-2-0181, and ONR Grant N00014-20-1-2822.
All authors are with the GRASP Laboratory,
        University of Pennsylvania, United States
        {\tt\small \{igorspas, liuxu, aribeiro, pappasg, kumar\}@seas.upenn.edu}.}%
}

\maketitle

\begin{abstract}

Accurate and robust state estimation is critical for autonomous navigation of robot teams. %
This task is especially challenging for large groups of size, weight, and power (SWAP) constrained aerial robots operating in perceptually-degraded GPS-denied environments. %
We can, however, actively increase the amount of perceptual information available to such robots by augmenting them with a small number of more expensive, but less resource-constrained, agents.
Specifically, the latter can serve as sources of perceptual information themselves. %
In this paper, we study the problem of optimally positioning (and potentially navigating) a small number of more capable agents to enhance the perceptual environment for their lightweight, inexpensive, teammates that only need to rely on cameras and IMUs. %
We propose a numerically robust, computationally efficient approach to solve this problem via nonlinear optimization. %
Our method outperforms the standard approach based on the greedy algorithm, while matching the accuracy of a heuristic evolutionary scheme for global optimization at a fraction of its running time. %
Ultimately, we validate our solution in both photorealistic simulations and real-world experiments. %
In these experiments, we use lidar-based autonomous ground vehicles as the more capable agents, and vision-based aerial robots as their SWAP-constrained teammates. %
Our method is able to reduce drift in visual-inertial odometry by as much as $\mathbf{90 \%}$, and it  outperforms random positioning of lidar-equipped agents by a significant margin. %
Furthermore, our method can be generalized to different types of robot teams with heterogeneous perception capabilities. %
It has a wide range of applications, such as surveying and mapping challenging dynamic environments, and enabling resilience to large-scale perturbations that can be caused by earthquakes or storms. %

\end{abstract}
\section{Introduction}

The low size, weight, and power requirements of inertial measurement units (IMUs) and cameras have made them a standard combination of sensors for SWAP-constrained flying robots such as agile micro aerial vehicles.
Nevertheless, harnessing these sensors for accurate state estimation requires overcoming several challenges.
In particular, visual-inertial odometry (VIO) quickly accumulates significant drift in unstructured, dynamic, or featureless environments. 
Such drift can dramatically degrade the efficacy and safety of an autonomous system.

Effective use of cameras rests upon the presence of a sufficient quantity of visual information in the form of texture or landmarks in the environment. 
Information-rich scenarios display a large set of visual cues, typically in the form of landmarks spread throughout the surroundings of the agent. 
In the absence of such visual features, the quality of state estimates rapidly degrades.
To overcome this challenge, we consider augmenting low-SWAP robots with a small number of agents without such stringent payload constraints to robustly navigate unstructured environments.

\begin{figure}[t!]
\centering
\includegraphics[trim=0 0 0 0, clip,width=0.9\columnwidth]{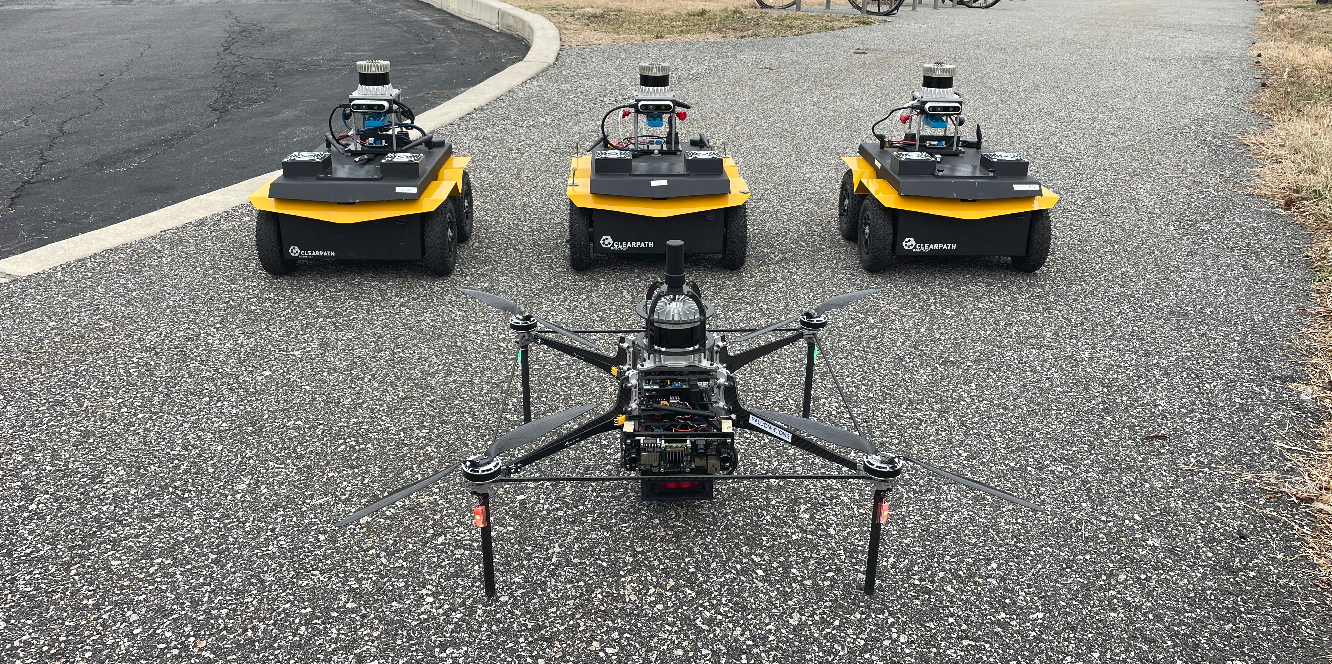}
    \caption{Robot platforms used in our experiments. Both platforms are capable of long-range GPS-denied autonomous navigation. In the experiments of this paper, the UAV relies only on cameras and the IMU for state estimation, while the UGVs rely on lidars. We refer readers to \cite{liu2022large} and \cite{miller2022stronger} for details on our UAV and UGV platforms.}
    \label{fig:robots}
\end{figure}

The predominant way of solving such problems for a single agent involves planning perception-aware trajectories. %
Intuitively, this means choosing paths that only traverse regions of the environment with sufficiently rich perceptual content. %
Naturally, this limits the operating space of the robot. %
However, in multi-agent settings, we can leverage the fact that we can design heterogeneous teams comprising of agents with different perceptual capabilities. %
In particular, a select subset of just a few perceptually-advantaged robots can act as landmarks for their vision-driven teammates. %
Our approach involves optimizing the positions of the former subset of agents using a second-order smooth optimization scheme.

We summarize our contributions as follows: %
\begin{enumerate}
    \item We propose a framework to tackle the multi-robot localization problem, by letting select members of the team serve as landmarks for their teammates with different sensing modalities. %
    \item We propose an algorithm for the underlying active robot positioning problem that outperforms an approximate method based on a greedy algorithm. Furthermore, numerical results suggest that our algorithm matches the accuracy of a heuristic evolutionary scheme for global optimization while having a significantly lower computational demand. 
    \item We validate our method in photorealistic simulations and real-world experiments using a robot team composed of one UAV and multiple UGVs as shown in \cref{fig:robots}. In particular, we show that our method can reduce VIO drift by as much as 90\%. 
\end{enumerate}

The rest of the paper is organized as follows. %
Section \ref{sec: related work} summarizes the related work. %
Section \ref{sec:problem-form} gives a detailed problem formulation, followed by Section \ref{sec:method} in which we present our approach. 
The analysis of our algorithm is presented in Section \ref{sec: algorithm-analysis},
and its numerical performance is showcased in Section \ref{sec:numerical-analysis}. %
The simulation and real-world experiments in Section \ref{sec:experiments} ultimately demonstrate the efficacy of our method. %

\section{Related Work}
\label{sec: related work}

\subsection{Vision-based State Estimation}
Vision-based state estimation has been maturing and gaining popularity during the past decade. Powerful monocular-based state estimation algorithms such as the classical structure from motion algorithm and recent state-of-the-art monocular odometry methods \cite{newcombe2011dtam, engel2014lsd} can estimate camera poses and 3D structures. However, such algorithms cannot be directly used for robot navigation since the absolute scale of the world is not observable with a single camera. VIO algorithms can estimate poses in metric scale and run up to the IMU rate. Therefore, they are commonly used in robotics applications \cite{delmerico2018benchmark}. State-of-the-art VIO algorithms \cite{qin2018vins, sun2018robust, mur2015orb, campos2021orb} have robust performance in high-speed 3D navigation with aggressive motions \cite{mohta2018fast, mohta2018experiments, lin2018autonomous, zhou2022swarm}. 
  
However, pure geometric-based methods have some intrinsic limitations: (1) The storage demand is high when maintaining a geometric map over a long trajectory. (2) Geometric-based features are sensitive to changes in viewpoint or lighting conditions. (3) It is difficult to distinguish features extracted from dynamic and static objects, leading to failures in dynamic environments. As a result, they can accumulate significant drift over long trajectories, which leads to unsafe behaviors and errors in mapping. Learning-enabled approaches are used to improve state estimation and mapping. Among them, one popular choice is to use semantic features because they are distinctive, informative, memory-efficient, and robust to viewpoint changes \cite{bao2011semantic, salas2013slam++, bowman2017probabilistic, nicholson2018quadricslam, yang2019cubeslam, shan2020orcvio, chen2020sloam}. The proposed work can be categorized into semantic-based state estimation, where the semantic landmarks are UGVs.

\subsection{Active Perception}

Our work falls within the class of active perception and state estimation \cite{bajcsy1988} problems. 
Some of the earliest methods modeled this task using the framework of partially observable Markov decision processes (POMDP) \cite{sondikPOMDP, kaelbling1998planning}.
This paradigm addresses synthesizing optimal policies, that is, functions mapping the history of observations of an agent to the optimal control for executing a given task. 
However, first examples of solutions to such problems were either developed for linear systems with Gaussian noise \cite{athansLQGhistory}, or problems with a small number of states \cite{kaelbling1998planning}.
Scaling this approach to realistic, and more complex robotic systems proved challenging. 
On the theoretical side, researchers showed that this should come as no surprise \cite{complexityofMDPs}: solving \textit{generic} POMDPs is a PSPACE-hard problem. 
Furthermore, they showed that even with infinite precompute, an optimal controller, in general, might be challenging to implement in an efficient algorithm. 

As a result, the community turned to solving approximate versions of the problem. 
The hope was that by acting optimally according to a proxy objective, perception-aware behavior would emerge. 
One of the first classes of methods involved belief space planning \cite{prentice2009beliefroadmap, bry2011rapidly, van2011lqgmp, van2012motion, charrow2015information, atanasov2015decentralized, sun2020stochastic}. 
Effectively, this meant solving the open-loop planning version of the more difficult POMDP. 
Somewhat unfortunately, \cite{complexityofMDPs} also showed that solving this problem in general is NP-hard. 
To date, the majority of approaches based on BSP have been computationally intensive, and often challenging to implement for real-time use.

Recently, there has been a surge in interest in perception-aware motion planning for agile robots. 
Roughly speaking, this paradigm optimizes trajectories of agents, while ensuring they maintain a sufficient amount of perceptual information within the sensing region of their onboard sensors \cite{falanga2018pampc, zhang2020fisher, tordesillas2022panther, spasojevic2020perception, murali2019acc}. 
However, these approaches focus on single-agent settings. 
Their environment is fixed. 
When generalizing the methods to perception-aware navigation for robot teams, 
naturally, the same algorithms can be used for every member of the team. 
However, none of these works adopt the collaborative approach of altering the perceptual content of the environment created by the agents themselves. 
This is precisely the theme of our paper.

\subsection{Active Team Localization and Target Tracking}

There has been a substantial amount of work on the localization of robot teams using relative observation measurements. 
Many different versions of the problem exist \cite{martinelli2005multi, FranchiMutualLocalizationIROS09, prorok2014accurate, nguyen2020vision, gaoMutualLocalizationRAL22} depending on the nature of allowed communication between teammates, as well as the type of measurement models (e.g. relative position, range, or bearing, etc...). 
Naturally, nonlinear measurement models render the resulting inference problem more computationally challenging. 
Recently, a number of papers have proposed robust algorithms based on convex relaxations that appear to be lossless for a non-negligible regime of noise levels \cite{tian2021distributed, rosen2022distributed}.

Prior work has also addressed the problem of multi-robot active SLAM \cite{kontitsis2013multi, atanasov2015decentralized, chen2020broadcast}, 
however, with several differences. 
For example, \cite{atanasov2015decentralized} does not leverage mutual observation of robots to aid localization. \cite{chen2020broadcast} relies on features in the environment to derive relative measurements between robots, a step that can fail in challenging, featureless environments. \cite{kontitsis2013multi} assumes constant communication and perfect data association, which are not always satisfied in real-world settings. Most importantly, none of the former consider actively positioning landmarks to help robots localize themselves.
The work of \cite{kurazume1994cooperative} considers a problem directly related to ours, albeit with a small team of 
homogeneous robots forming a predefined shape in 2D space. Subsequent works such as \cite{trawny2004optimized, tully2010leap} extend the former, but provide no empirical or theoretical sub-optimality results. 
In this work, we consider a fine-grained measure of localization accuracy of a heterogeneous robot team based on a minimax localization cost.
Furthermore, unlike previous works, we empirically show that a suitable gradient-descent-based algorithm matches the output of an evolutionary method for positioning perceptually-advantaged agents with significant accuracy.

Last but not least, our problem is intimately related to that of target tracking \cite{spletzer2003dynamic, lifeng23trackingTRO, lifengtzoumas19RAL}. 
The latter problem involves positioning the set of tracker agents to allow them to form the best estimate of the target states. 
The parallel to our problem amounts to identifying tracking agents with lidar-equipped ones, and the targets with vision-aided agents. 
Typically, tracker agents communicate among themselves in order to leverage their collective measurements to position themselves optimally. 
Our setup effectively considers the dual of this problem. 
Other notable differences include the fact that we consider a collaborative setting, in addition to allowing no communication between agents except at the start of the mission.

\subsection{Sensor Selection}

Optimally placing the small set of lidar-equipped agents can be viewed as a sensor selection problem. 
Intuitively, this involves choosing out of a set of possible sensor locations the subset that gives minimal uncertainty in the estimates of the unknown quantity being measured subject to suitable, typically budget, constraints. 
In our example, the set is in fact not discrete, but comprises a whole continuum of possible positions in the environment in which we can place the lidar-equipped agents. 
The budget comes simply from the number of such agents at our disposal. 
It has been shown that approximately solving the discrete version of the sensor selection problem to an arbitrary level of suboptimality is NP-hard \cite{feige1998threshold, olshevsky2014minimal}. 
Notable approaches based on the greedy paradigm that have nevertheless worked well in practice include \cite{krause2008near, jawaid2015submodularity, carlone2018attention, tzoumas2020lqg, chamon2020approximate}. 
However, we adopt a different approach based on nonlinear optimization.
\section{Active Collaborative Team Localization}
\label{sec:problem-form}

\begin{figure}[h!]
\centering
\includegraphics[trim=0 0 0 0, clip,width=0.75\columnwidth]{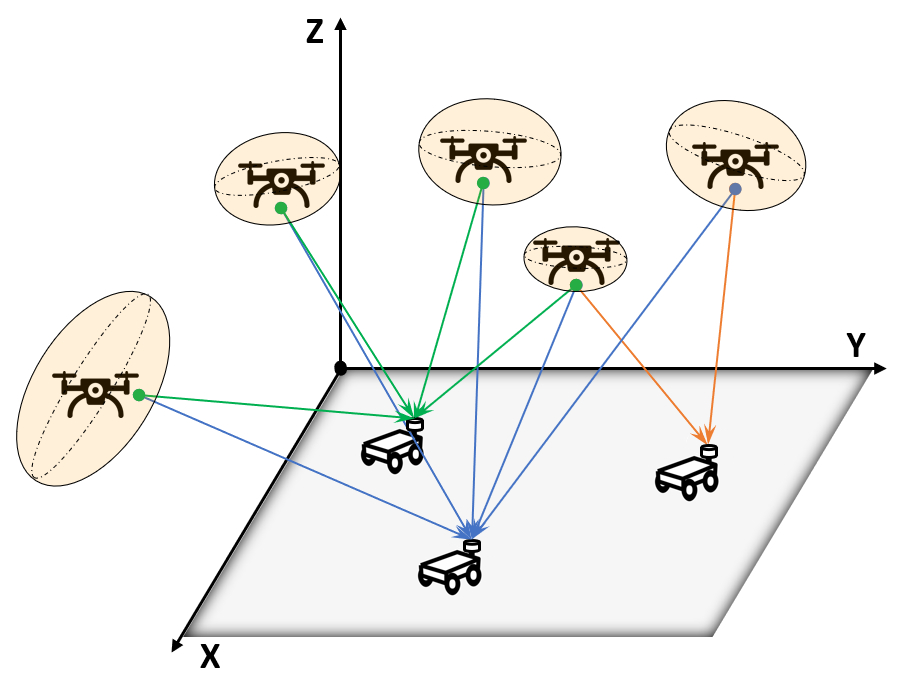}
    \caption{In the absence of reliable visual features or landmarks in the environment, $\mathcal{C}$-agents (UAVs) additionally triangulate their positions using bearing measurements to $\mathcal{L}$-agents (UGVs). We position $\mathcal{L}$-agents to minimize the localization uncertainty (ellipses) of $\mathcal{C}$-agents.}
    \label{fig:robot-team-problem-form}
\end{figure}


Our robot team is partitioned into two subsets of agents, $\mathcal{L}$ and $\mathcal{C}$. 
Agents in set $\mathcal{L}$, also referred to as ``$\mathcal{L}$-agents'', are expensive lidar-equipped agents that can autonomously position themselves accurately without additional external localization infrastructure. 
Set $\mathcal{C}$ comprises of inexpensive, resource-constrained agents (``$\mathcal{C}$-agents'') that predominantly rely on their onboard cameras for state estimation. 
However, cameras can be used to obtain accurate state estimates only when there is a sufficient quantity of reliable visual cues in the environment. 
We propose a solution for accurate team localization robust to whether or not the latter condition holds. 
In particular, our approach involves leveraging $\mathcal{L}$-agents as artificial landmark features $\mathcal{C}$-agents can use for state estimation in desired regions of space.

In light of the setup $|\mathcal{L}| \ll |\mathcal{C}|$ motivated by budget considerations, it is intuitively clear that certain configurations of $\mathcal{L}$-agents are better than others. 
For example, positioning them in degenerate configurations (such as a single point) would result in worse performance than spreading them further apart.
On the other hand, placing $\mathcal{L}$-agents too far would supply $\mathcal{C}$-agents with bearing measurements with near-vanishing sensitivity to change in their positions. 
Furthermore, tailoring the positioning to enable accurate vision-aided localization in one region of space might render it useless for state estimation in another.

We therefore center our approach around a minimax optimization problem stated as follows.
Given $N$ representative points of operating regions of $\mathcal{C}$-agents (without loss of generality one point per each $\mathcal{C}$-agent so that $|\mathcal{C}| = N$), and a given budget on the number of $\mathcal{L}$-agents, say $|\mathcal{L}| \leq M$, position the $\mathcal{L}$-agents in a way that minimizes the maximum localization uncertainty across all $\mathcal{C}$-agents.
Identifying $\mathcal{C}$ and $\mathcal{L}$ with $[N]$ and $[M]$, respectively, and denoting the locations of $\mathcal{C}$-agents and $\mathcal{L}$-agents by $(\mathbf{x}_i)_{i \in [N]} \in (\mathbb{R}^3)^N$ and $(\mathbf{z}_j)_{j \in [M]} \in (\mathbb{R}^3)^M$, respectively, and the covariance of the resulting state estimate at location $\mathbf{x}_i$ by $\Sigma(\mathbf{x}_i; \ \mathbf{z}_{1:M})$, we have  
\begin{problem} \label{prob:problem_one}
\begin{equation*}
\begin{aligned}
\min_{z_{1:M} \in (\mathbb{R}^3)^M} \ & \max_{i \in [N]} \ tr( \ \Sigma(\mathbf{x}_i; \ \mathbf{z}_{1:M}) \ ) \\
s.t. & \\
& || \mathbf{z}_j ||_2 \leq R_{max} \hspace{9mm} \forall j \in [M] \\
& || \mathbf{z}_j - \mathbf{x}_i ||_2 \geq R_{min} \quad \forall i \in [N], \ \forall j \in [M] \\
\end{aligned}
\label{eqn:original-problem}
\end{equation*}
\end{problem}
The latter two constraints encode the desideratum that we can only position $\mathcal{L}$-agents in a certain region of space, as well as the fact that no $\mathcal{C}$-agent can be closer than a minimum prescribed distance to any $\mathcal{L}$-agent. 
We measure the localization accuracy of each $\mathcal{C}$-agent using the trace of the covariance of its state estimate.
This roughly captures its total uncertainty across all three spatial dimensions. 
Overall, our problem is novel in two ways:
\begin{enumerate}
    \item we propose the paradigm of enabling predominantly vision-driven robot teams to navigate perceptually-challenging environments by using select perceptually-advantaged teammates as visual landmarks; 
    \item furthermore, we model the task using a minimax active robot positioning problem in continuous space; previous work on the topic has typically focused on its discrete version, without a robust minimax objective. 
\end{enumerate}

\textbf{Remark.}
Two special cases of our setup deserve particular mention. 
One involves positioning $\mathcal{L}$-agents in order to ensure high localization accuracy of $\mathcal{C}$-agents along pre-planned trajectories. 
This scenario is subsumed by Problem \ref{prob:problem_one}, as can be seen by placing fictitious $\mathcal{C}$-agents (Fig \ref{fig:traj_planning}) at regular ``discretization'' points along paths to be taken by ``real'' $\mathcal{C}$-agents. 
The second involves using $\mathcal{L}$-agents to provide visual cues at a well-dispersed sample of positions for the purpose of coverage of a given environment.
The latter instantiation of Problem \ref{prob:problem_one} amounts to placing fictitious $\mathcal{C}$-agents at sampled points.

\begin{figure}[h!]
\centering
\includegraphics[trim=0 0 10 0, clip,width=1.0\columnwidth]{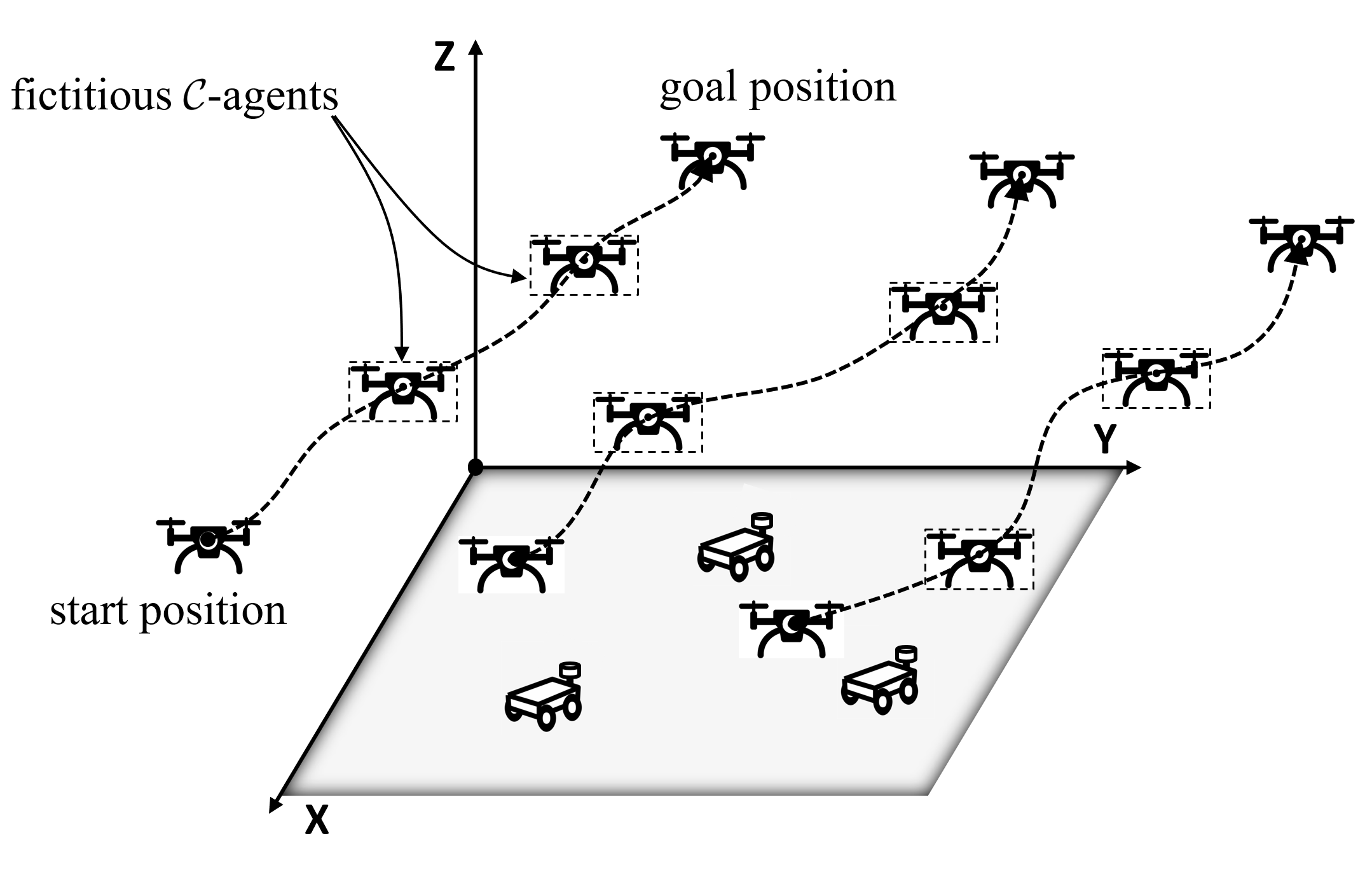}
\captionof{figure}{Placing sufficiently many \textit{fictitious} $\mathcal{C}$-agents along planned trajectories of real $\mathcal{C}$-agents to ensure the latter have accurate state estimates.}
\label{fig:traj_planning}
\end{figure}

Regarding communication, we assume the following protocol.
The team rallies together at the start of the mission, has one agent solve the problem in a centralized way, communicate the positions of $\mathcal{L}$-agents to the rest of the team, and then the mission starts. 
From that moment onward, no explicit communication between teammates takes place.

\section{Active Positioning Algorithm}
\label{sec:method}


\subsection{Modelling Sensing and Localization Uncertainty}

The observation model of each $\mathcal{C}$-agent comprises a set of bearing measurements. 
In particular, $\mathcal{C}$-agent $i$ at point $\mathbf{x} \in \mathbb{R}^3$ in the vicinity of setpoint $\mathbf{x}_i$ registers the bearing to $\mathcal{L}$-agent $j$ at position $\mathbf{z}_j \in \mathbb{R}^3$ in the form:
\begin{equation}
\mathbf{y}_{i,j} = \mathbf{h}(\mathbf{x}, \mathbf{z}_j) + \mathcal{N}(0, \ \sigma_m^2 I_3).
\label{eqn:noisy-bearing}
\end{equation}
Here $\mathbf{h}(\mathbf{x}, \mathbf{z}) := \frac{\mathbf{z} - \mathbf{x}}{|| \mathbf{z} - \mathbf{x} ||_2}$ denotes the noiseless bearing measurement function.  
We assume sensor noise is normally distributed with standard deviation $\sigma_m > 0$, and is independent across different $\mathcal{C}$-$\mathcal{L}$ agent pairs.

A critical aspect of our problem involves the \textit{nonlinear} sensing model in Equation \ref{eqn:noisy-bearing}. 
As a result, we now describe the way in which $\mathcal{C}$-agents translate such measurements into estimates of their pose. 
We take the route based on the paradigm of extended Kalman filtering (EKF). 
In particular, when the $\mathcal{C}$-agent at position $\mathbf{x}$ registers bearings to $\mathcal{L}$-agents at positions $\mathbf{z}_1, \mathbf{z}_2, ..., \mathbf{z}_M \in \mathbb{R}^3$, its collection of measurements may be succinctly represented via 
\begin{equation}
\mathbf{y}_{1:M} = \mathbf{H}(\mathbf{x}; \mathbf{z}_{1:M}) + \mathbf{\epsilon},
\end{equation}
where $\mathbf{H}(\mathbf{x}; \mathbf{z}_{1:M}) = [ \mathbf{h}^T(\mathbf{x}, \mathbf{z}_1),  \mathbf{h}^T(\mathbf{x}, \mathbf{z}_2), \dots , \mathbf{h}^T(\mathbf{x}, \mathbf{z}_M) ]^T$, and $\mathbf{\epsilon} = [\mathbf{\epsilon}_1^T, \mathbf{\epsilon}_2^T, \dots, \mathbf{\epsilon}_M^T]^T$ are i.i.d. Gaussian random variables per Equation \ref{eqn:noisy-bearing}.
Assuming a $\mathcal{N}(\mathbf{x}_i, \bar{\Sigma})$ prior on $\mathbf{x}$, the agent approximates the bearing measurement to agent $j$ via the linearized measurement model
\begin{equation}
\bar{\mathbf{y}}_{i,j} + \Delta \mathbf{y}_{i,j} = \underbrace{\mathbf{h}(\mathbf{x}_i, \mathbf{z}_j)}_{=: \bar{\mathbf{y}}_{i,j}} + \frac{\partial \mathbf{h}}{\partial x}(\mathbf{x}_i, \mathbf{z}_j) \underbrace{\Delta \mathbf{x}}_{=: \mathbf{x} - \mathbf{x}_i} + \mathbf{\epsilon_i},
\end{equation}
so that its posterior distribution is approximately Gaussian with the following information (inverse covariance) matrix
\begin{equation}
\mathcal{J}(\mathbf{x}_i; \ \mathbf{z}_{1:M}) = \bar{\Sigma}^{-1} + \sigma_m^{-2} \sum_{j = 1}^M \frac{\partial \mathbf{h}^T}{\partial x}(\mathbf{x}_i, \mathbf{z}_j) \frac{\partial \mathbf{h}}{\partial x}(\mathbf{x}_i, \mathbf{z}_j).
\label{eqn:total-info-matrix}
\end{equation}
Recalling the objective function of Problem \ref{prob:problem_one}, we define
\begin{equation} \label{eqn:EKF_covariance}
\Sigma( \mathbf{x}_i; \ \mathbf{z}_{1:M} ) =  \mathcal{J}^{-1}(\mathbf{x}_i; \ \mathbf{z}_{1:M}) .
\end{equation}
Regarding the sensing model and pose uncertainty of every $\mathcal{C}$-agent, we make two assumptions hereon:
\begin{enumerate}
    \item $\mathcal{C}$-agents have accurate orientation estimates;
    \item their onboard cameras are omnidirectional.
\end{enumerate}
\textbf{Remarks}
For the first assumption, we follow the approximation made in \cite{carlone2018attention}. 
Intuitively, it posits that if bearing measurements arrive at a sufficiently high frequency (approximately every 5 seconds in our experiments), the gyroscope can be used to accurately estimate incremental changes in orientation between successive keyframes.
At latter points, visual information is fused with IMU measurements to bring the orientation uncertainty down to a negligible value.
The second assumption is valid when $\mathcal{C}$-agents are equipped with multiple cameras that together cover a $4\pi$ (in steradians) field of view (FOV). 
It can also be \textit{emulated} by agents equipped with a camera with a limited FOV by periodically performing a $2\pi$ yawing motion at specified setpoints. 
In our physical experiments, we take the latter route.

\subsection{Optimization Algorithm}

Problem \ref{prob:problem_one} is a challenging nonlinear optimization task. 
Nevertheless, several practical approaches come to mind. 
One method involves discretizing the continuous problem with a grid of points sampled from its feasible (``obstacle-free'') region.
Starting from an empty set, this algorithm then greedily picks a subsequent grid point that induces the maximum decrease in the objective value, until a given number of $\mathcal{L}$-agents have been positioned.
A drawback of this approach is that it does not result in provably globally optimal solutions, while incurring a non-negligible computational cost - typically linear in the number of discretization points.
The latter generally has to increase with the desired level of accuracy. 
Another method involves a heuristic evolutionary procedure for global optimization \cite{globalOptimizationAlgo}. 
In spite of promising empirical performance, this algorithm can demand long compute times and it does not come with global optimality guarantees. 
Finally, the third class of approaches includes local gradient-descent based algorithms for smooth optimization. 
This class of methods has shown impressive results, scaling to high-dimensional optimization problems involving training neural networks. 
Nevertheless, their computational performance can vary significantly depending on the exact algorithm being used, as well as the structural properties of the problem being solved.

Our method belongs to the class of gradient-descent-based approaches.
However, further specific design choices are still required. 
In particular, Problem \ref{prob:problem_one} is a non-smooth non-convex optimization problem. 
Lack of global smoothness stems from the minimax form of the objective function that can dramatically slow down the rate of convergence of local algorithms for smooth optimization.
The other challenging aspect of Problem \ref{prob:problem_one} is its non-convexity. 
Indeed, we show there exist instances of the problem that exhibit multiple strict local optima with wildly different objective values. 
One such example is shown in Fig. \ref{fig:landscape}. 
In what follows, we show how to tackle each of these hurdles in turn. 

\begin{figure}[h!]
\centering
\includegraphics[scale=0.3]{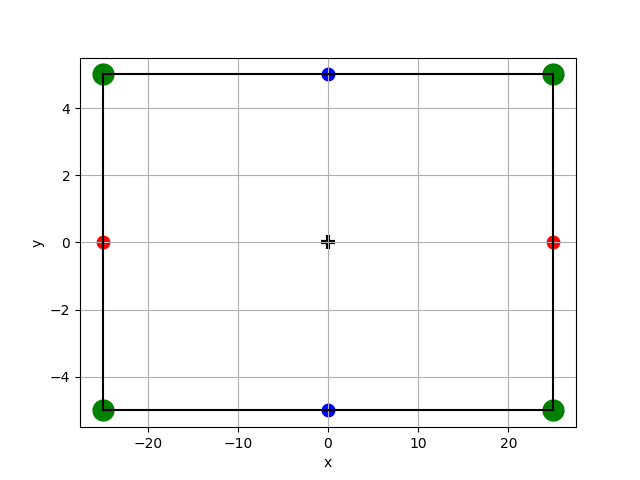}
    \caption{Two strict local minima of the objective function of the problem involving minimizing the maximum localization uncertainty of four $\mathcal{C}$-agents operating at setpoints depicted by green circles with the help of two $\mathcal{L}$-agents. The blue local optimum 
    is more than 
    $500\%$ worse than the red local optimum. The cross at the centroid of the rectangle is a saddle point. }
    \label{fig:landscape}
\end{figure}

To ensure our objective function is globally smooth, we introduce two modifications to the original problem. 
The first involves choosing a positive real number $\delta \ll R_{min}$, and modifying the noiseless bearing measurement function via:
\begin{equation} \label{eq:smooth_bearing}
\mathbf{h}(\mathbf{z}, \mathbf{x}) =  \frac{\mathbf{z} - \mathbf{x}}{|| \mathbf{z} - \mathbf{x} ||_2} \quad \leadsto \quad \mathbf{h}_{\delta}(\mathbf{z}, \mathbf{x}) =  \frac{\mathbf{z} - \mathbf{x}}{\sqrt{|| \mathbf{z} - \mathbf{x} ||_2^2  + \delta^2}}.
\end{equation}
Using the same methodology of extended Kalman filtering, we make the corresponding changes to the posterior covariance estimates of agents at corresponding setpoints. 
For example, the localization uncertainty of the $\mathcal{C}$-agent operating at setpoint $i$ is modified through
\begin{equation} 
\begin{aligned}
\mathcal{J}_{\delta}(\mathbf{x}_i; \ \mathbf{z}_{1:M}) & = \bar{\Sigma}^{-1} + \sigma_m^{-2} \sum_{j = 1}^M \frac{\partial \mathbf{h}_{\delta}^T}{\partial x}(\mathbf{x}_i, \mathbf{z}_j) \frac{\partial \mathbf{h}_{\delta}}{\partial x}(\mathbf{x}_i, \mathbf{z}_j) \\
\Sigma_{\delta}( \mathbf{x}_i; \ \mathbf{z}_{1:M} ) & =  \mathcal{J}_{\delta}^{-1}(\mathbf{x}_i; \ \mathbf{z}_{1:M}). \\
\end{aligned}
\end{equation}
As a result, we solve
\begin{problem} \label{prob:problem_delta}
\begin{equation*}
\begin{aligned}
\min_{z_{1:M} \in (\mathbb{R}^3)^M} \ & \max_{i \in [N]} \ tr( \ \Sigma_{\delta}(\mathbf{x}_i; \ \mathbf{z}_{1:M}) \ ) \\
s.t. & \\
& || \mathbf{z}_j ||_2 \leq R_{max} \hspace{9mm} \forall j \in [M] \\
& || \mathbf{z}_j - \mathbf{x}_i ||_2 \geq R_{min} \quad \forall i \in [N], \ \forall j \in [M] \\
\end{aligned}
\label{eqn:original-problem}
\end{equation*}
\end{problem}
The second modification involves reformulating the minimax Problem \ref{prob:problem_delta}, in general one with non-smooth objective, into a smooth problem. 
We introduce a new decision variable $t \in \mathbb{R}$, and consider the following equivalent 
\begin{problem} \label{prob:epigraph_delta_problem}
\begin{equation} 
\begin{aligned}
\min_{\mathbf{z}_{1:M} \in (\mathbb{R}^3)^M, \ t \geq 0} \ & t \\ 
s.t. & \\
& tr( \ \Sigma( \mathbf{x}_i; \ \mathbf{z}_{1:M} ) \ ) \leq t \quad \forall i \in [N] \\
& || \mathbf{z}_j ||_2 \leq R_{max} \hspace{9mm} \forall j \in [M] \\
& || \mathbf{z}_j - \mathbf{x}_i ||_2 \geq R_{min} \quad \forall i \in [N], \ \forall j \in [M]
\end{aligned}
\end{equation}
\end{problem}
Intuitively, the fact that $t$ is being minimized in conjunction with the first set of constraints will effectively make it take the value of the objective of Problem \ref{prob:problem_delta}.
Problem \ref{prob:epigraph_delta_problem} is still a non-convex optimization problem. 
For this reason, we solve it using an interior point solver, IPOPT \cite{wachter2006implementation}, to which we pass analytic derivatives of constraints of zeroth, first, and second orders. 
Here we leverage the fact that the function 
\begin{equation}
J \rightarrow tr( J^{-1} )
\end{equation}
is a smooth map on the space of positive \textit{definite} matrices (in this case in $\mathbb{R}^{3 \times 3}$).
Furthermore, assuming that $J$ is a smooth, positive definite function of a set of parameters $\theta \in \mathbb{R}^k$, we have the following relations 
\begin{equation}
\nabla_{\theta_i} \ tr(J^{-1}(\theta)) = (-1) tr( J^{-2}(\theta) \frac{\partial J}{\partial \theta_i}(\theta) ),
\end{equation}
and 
\begin{equation}
\begin{aligned}
\nabla_{\theta_i, \theta_j}^2 \ tr(J^{-1}(\theta)) & = \\
 - tr( J^{-2}(\theta) \frac{\partial^2 J}{\partial \theta_i \partial \theta_j}(\theta)) & + 2 tr( J^{-2}(\theta) \frac{\partial J}{\partial \theta_i}(\theta) J^{-1}(\theta) \frac{\partial J}{\partial \theta_i}(\theta)  )
\end{aligned}
\end{equation}

\section{Algorithm Analysis}
\label{sec: algorithm-analysis}

\subsection{Price of Smoothing}

We now elaborate on our choice of $\delta > 0$ in Equation \ref{eq:smooth_bearing}.
Indeed, any $\delta > 0$ introduces fictitious information about the range from a $\mathcal{C}$-agent at position $\mathbf{x}$ to a $\mathcal{L}$-agent at position $\mathbf{z}$. 
This follows from the fact that the directional derivative of $\mathbf{h}_{\delta}$ along the bearing vector $\frac{\mathbf{z} - \mathbf{x}}{|| \mathbf{z} - \mathbf{x} ||_2}$ is \textit{non-vanishing} whenever $\delta > 0$. 

However, at a fixed range (i.e. fixed $|| \mathbf{z} - \mathbf{x} ||_2$), the influence of $\delta$ decreases as $\delta \downarrow 0$. 
Intuitively, this implies that at any point in feasible regions of Problems \ref{prob:problem_one} and \ref{prob:problem_delta} (which coincide), the values of the objective functions of the two problems match each other in the limit $\delta \downarrow 0$. 
At this stage two questions remain:
\begin{enumerate}
    \item determine a suitable range (i.e. upper bound) of $\delta$
    \item quantify the suboptimality of solving Problem \ref{prob:problem_delta} in place of Problem \ref{prob:problem_one} as a function of $\delta$.
\end{enumerate}
To answer these questions, we develop a result in the active robot positioning problem that is robust to measurement model mismatch. 
The key result that enables such analysis is the following

\begin{theorem} \label{theorem:theorem_multiplicative}
Consider the setup of Problems \ref{prob:problem_one} and \ref{prob:problem_delta}. 
Let $\eta > 0$ and $\zeta > 1$ be an arbitrary pair of real numbers. 
Define 
\begin{equation}
s_0 = \frac{\eta}{\eta + 1} \frac{R_{min}^2}{M} \sigma_m^2 (1 + \zeta^2).
\label{eqn:allowed-s}
\end{equation}
Then, choosing any $\delta \in (0, R_{min}/\zeta)$, we have the implication
\begin{equation}
\begin{aligned}
\forall \mathbf{x} \in \mathbb{R}^3 \ & : \ tr( \mathcal{J}_{\delta}(\mathbf{x} ; \mathbf{z}_{1:M})^{-1} ) = s \leq s_0  \\
\quad 
& \Rightarrow 
\quad 
\frac{s}{1 + \eta} \leq tr( \mathcal{J}_{0}(\mathbf{x} ; \mathbf{z}_{1:M})^{-1} ) \leq s (1 + \eta).
\end{aligned}
\end{equation}
Furthermore, the same implication holds with the roles of $\mathcal{J}_{\delta}$ and $\mathcal{J}_{0}$ reversed.
\end{theorem}

\begin{proof}
See Appendix in Section \ref{sec:appendix_a}.
\end{proof}

\begin{corollary} \label{cor:guarantee_corollary}
Within the setup of Theorem \ref{theorem:theorem_multiplicative}, choose any $\delta \in (0, R_{min}/\zeta)$.
Denote the values of Problems \ref{prob:problem_one} and \ref{prob:problem_delta} by $V^{0}$ and $V^{\delta}$, respectively. 
Suppose the sublevel set 
\begin{equation}
\mathcal{S} := \{ \mathbf{z}_{1:M} \ | \ V^{\delta}(\mathbf{z}_{1:M}) \leq s_0 \}
\end{equation}
is not empty. 
Then the optimal solution of Problem \ref{prob:problem_delta} is at most a $(1 + \eta)^2$-suboptimal solution of Problem \ref{prob:problem_one}.
\end{corollary}

\begin{proof}
See Appendix in Section \ref{sec:appendix_a}.
\end{proof}

The significance of Corollary \ref{cor:guarantee_corollary}, is that if there exists a solution to the Problem \ref{prob:problem_delta} with a sufficiently low objective value, then that solution is also approximately optimal for the original problem. 
To choose an appropriate value of $\zeta$, we set $s_0$ to be the square of the largest $\sigma$-uncertainty ball of each of the $\mathcal{C}$-agents that we are willing to tolerate. 
Naturally, this quantity can depend on the scale of the environment, and its density of obstacles. Here we call it $R_{tol}$. 
We then set
\begin{equation}
R_{tol}^2 \leq \frac{\eta}{\eta + 1} \frac{R_{min}^2}{M} \sigma_m^2 (1 + \zeta^2), 
\end{equation}
and so it is enough to choose
\begin{equation}
\zeta^2 = \left( \frac{R_{tol}}{R_{min}} \right)^2 M \sigma_m^{-2} \frac{2}{\eta}.
\end{equation}
Some ``physically reasonable'' values in the expression above involve setting
\begin{equation}
\frac{R_{tol}}{R_{min}} \rightarrow 10, \ M \rightarrow 10, \sigma_m \rightarrow 10^{-2}, \eta \rightarrow 10^{-1}
\end{equation}
thus leading to
\begin{equation}
\zeta^2 = 2 \times 10^9.
\end{equation}
In practice, we cap $\zeta$ to $10^{5}$ out of concern for numerical precision.

\subsection{Computational Complexity}

The computational complexity of the algorithm is determined by two factors. 
Firstly, it depends on the complexity of computing derivatives of the desired order.
In our case, all derivatives of order up to two are computed in $O(M^2N)$ time.  
Second, both the number of iterations, as well as the effort of computing gradient descent steps within the IPM influence the running time of the algorithm. 
The total number of variables and constraints is $O(MN)$, and so, roughly speaking, every gradient descent step is computed by inverting a square matrix of dimension $O(MN)$. 
For dense matrices, this is done in time at most $O((MN)^3)$. 
However, IPOPT can exploit the sparsity pattern of the constraints to solve such systems more efficiently. Nevertheless, we leave more refined (favorable) estimates of running time for future work. 
Assuming the total number of iterations is $K$, the running time of the algorithm is of order at most $O(K(MN)^3)$. 

\subsection{Limitations}
\label{sec:algorithm-limitations}

Despite encouraging performance on random data (see Table \ref{table:size-accuracy}), our numerical algorithm is not perfect. 
Firstly, as a local-optimization-based method, it can get stuck in suboptimal local minima, such as in scenarios shown in Fig. \ref{fig:landscape}.
Second, our algorithm is a centralized solution. 
Although it empirically scales favourably with the size of robot teams we consider, we note that the current method cannot handle scenarios with an a priori unknown number of spatially distributed $\mathcal{C}$-agents. 
We leave addressing these limitations for future work.

\begin{figure*}[tp]%
\centering
\hspace{-15mm}
\begin{subfigure}{.4\columnwidth}
\includegraphics[scale=0.3, trim={15cm 4.2cm 15cm 0}, clip]{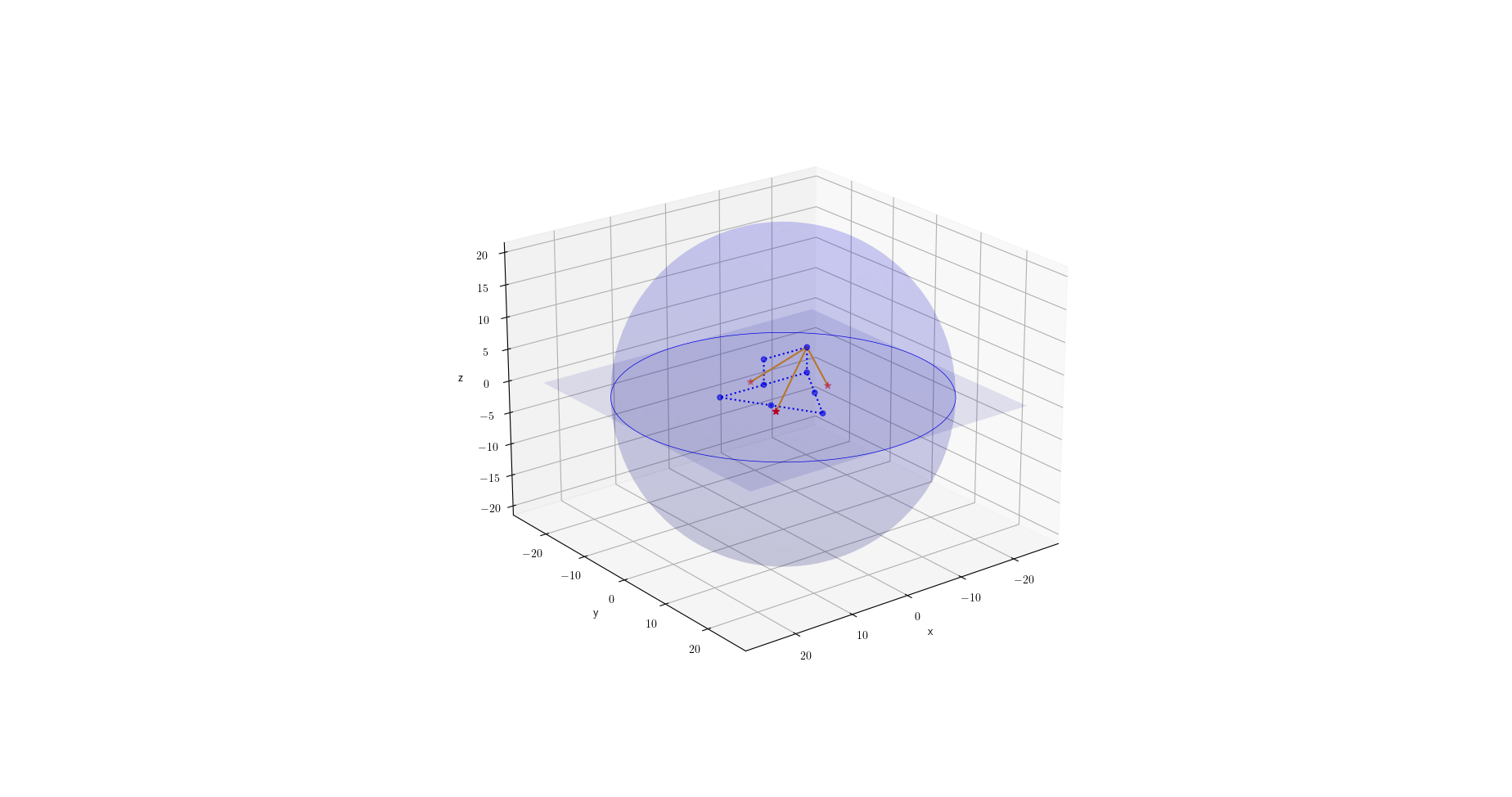}%
\caption{Eight $\mathcal{C}$-agents flying in formation.}%
\label{subfiga}%
\end{subfigure}\hspace{25mm}%
\begin{subfigure}{.4\columnwidth}
\includegraphics[scale = 0.3, trim={15cm 4.2cm 15cm 0}, clip]{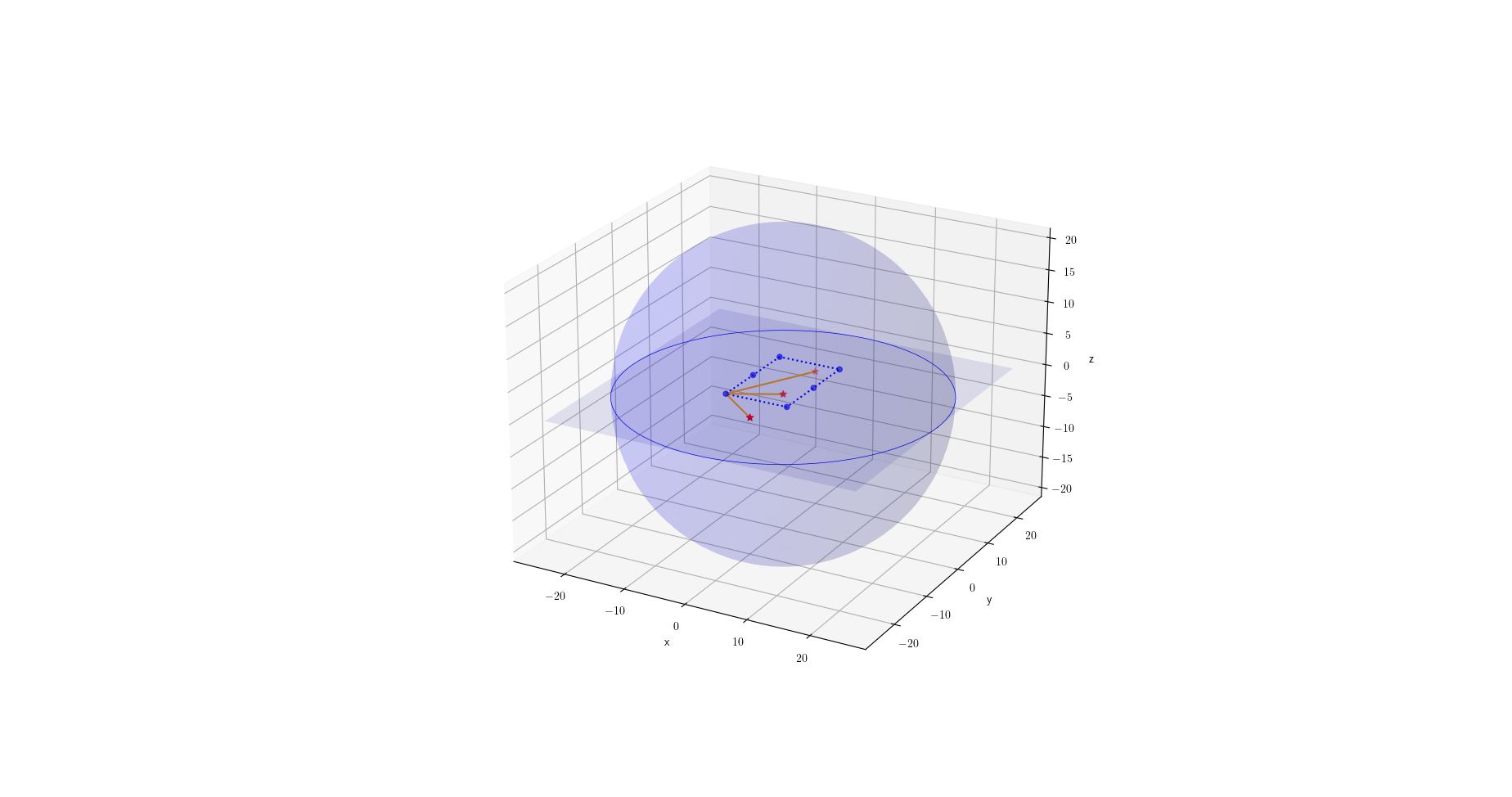}%
\caption{Six $\mathcal{C}$-agents on two straight line paths.}%
\label{subfigb}%
\end{subfigure}\hspace{25mm}%
\begin{subfigure}{.4\columnwidth}
\includegraphics[scale=0.3, trim={15cm 5cm 15cm 0}, clip]{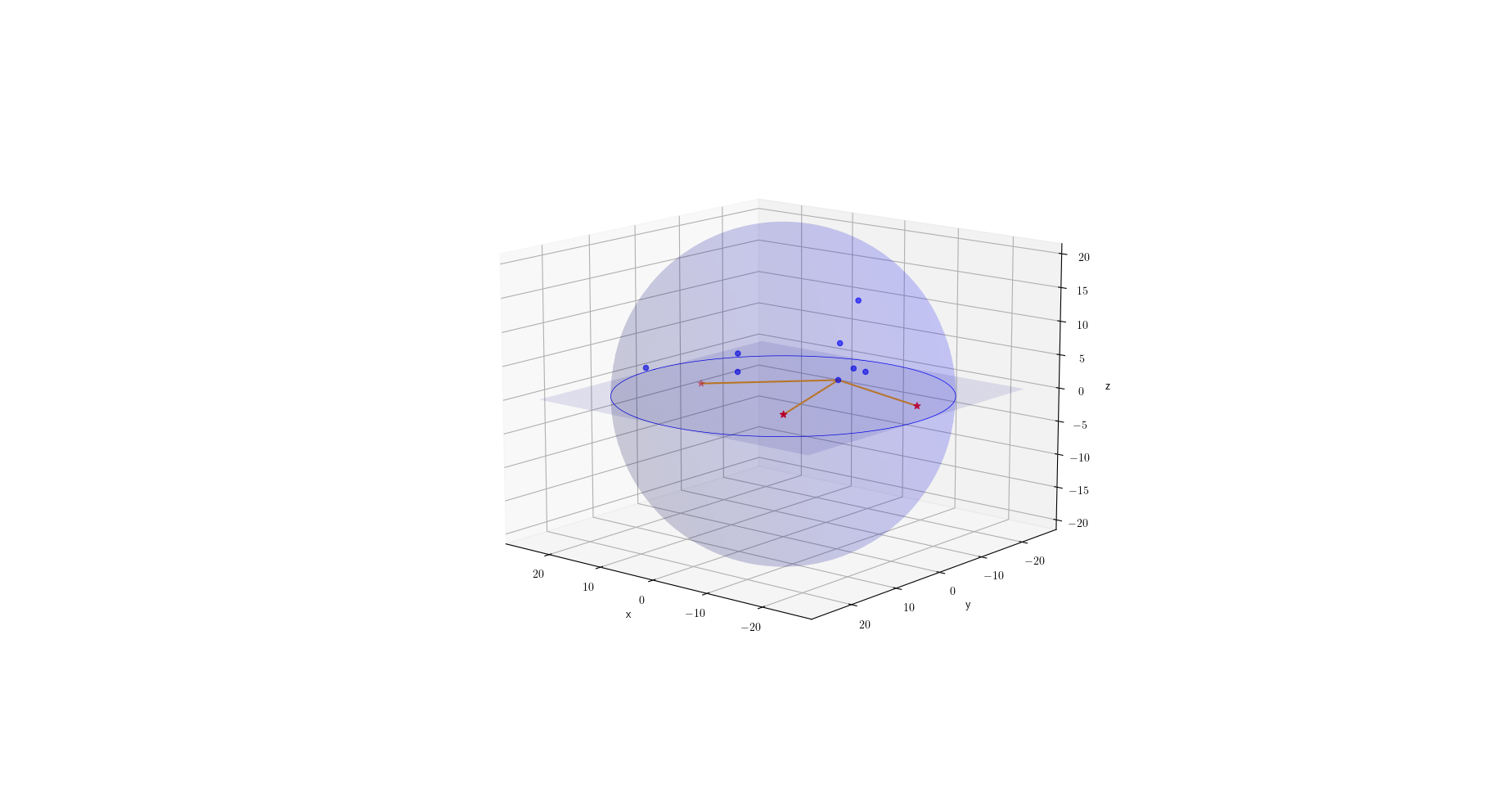}%
\caption{Eight $\mathcal{C}$-agents randomly positioned inside the upper hemisphere.}%
\label{subfigc}%
\end{subfigure}%
\caption{Examples of $\mathcal{L}$-agents positioned by our method. Blue dots are centered on each $\mathcal{C}$-agent/setpoint, while the red stars denote $\mathcal{L}$-agents. Orange lines show bearings from one of the $\mathcal{C}$-agents to all $\mathcal{L}$-agents.}
\label{figabc}
\end{figure*}

\section{Numerical Analysis}
\label{sec:numerical-analysis}

\begin{table*}[t!]
\centering
\begin{tabular}{ |c||c|c|c|c|c|c|c|c|c| }
\hline
$N$ & \multicolumn{3}{|c|}{$5$} & \multicolumn{3}{|c|}{$10$} & \multicolumn{3}{|c|}{$20$} \\
\hline
$M$ & $2$ & $5$ & $10$ & $2$ & $5$ & $10$ & $2$ & $5$ & $10$\\ 
\hline 
\hline
ours & $-0.002$ & $-0.01$ & $-0.02$ & $-0.001$ & $-0.01$ & $-0.03$ & $-0.002$ & $-0.01$ & $-0.03$\\ 
\hline 
greedy & $0.25$ & $0.48$ & $0.54$ & $0.23$ & $0.48$ & $0.52$ & $0.25$ & $0.37$ & $0.41$\\ 
\hline
\end{tabular}
\caption{The suboptimality of the approaches for various team sizes. We measure the suboptimality (loss of accuracy) of a solution in terms of its relative value compared to the objective achieved by the evolutionary optimization scheme.}
\label{table:size-accuracy}
\end{table*}

In this section, we test the numerical performance of our algorithm. 
We compare it against two baseline approaches, one a greedy approximate algorithm, and the other a heuristic evolutionary procedure for global optimization \cite{globalOptimizationAlgo} as implemented in the package NLOPT \cite{nlopt}. 
First, we look at the morphology of solutions generated by our algorithm, and then turn to comparing its accuracy and time efficiency compared to the other two methods. 


We begin by looking at the form of solutions of several robot team problems with the following parameters. 
The prior covariance is set to $\Sigma \rightarrow 30 \times I_3$, and the measurement standard deviation $\sigma_m \rightarrow 0.1$. 
We consider three problems, each involving three $\mathcal{L}$-agents, but varying numbers and configurations of $\mathcal{C}$-agents. 
Some characteristic solutions appear in \cref{figabc}. 

\subsection{Accuracy}

Here we show the comparison of the objective values achieved by our algorithm and the two alternate approaches on a range of different problems, differing in the number of agents, and the amount of noise in the sensor measurements. 

First, we set the noise to a set value, namely $\Sigma \rightarrow 30 \times I_3$ and $\sigma_m \rightarrow 0.1$, and vary the size of the team. 
In particular, we look at the performance of the algorithms as the number of $\mathcal{C}$-agents (denoted by $N$) and $\mathcal{L}$-agents (denoted by $M$) varies according to $N \in \{5, 10, 20\}$ and $M \in \{2, 5, 10\}$.
For every $(N, M)$ pair, we evaluate each algorithm on ten different random configurations of $\mathcal{C}$-agents, sampled from a sphere of radius $40$ m, and perform ten independent algorithm runs for each instance. Table \ref{table:size-accuracy} shows that the second-order method attains objectives matching those of the evolutionary algorithm on random data, and consistently outperforms the greedy approach - in some cases by as much as $50\%$.

Second, we test the accuracy of the algorithm on a problem with $N = 10$ $\mathcal{C}$-agents and $M = 5$ $\mathcal{L}$-agents, and varying amounts of sensing noise. 
For each level of noise, we run each algorithm on ten different random instances, generated as before. 
Table \ref{table:noise-accuracy} shows that our algorithm produces solutions that match the quality of those output by the evolutionary optimization scheme. 
Similarly as before, our algorithm consistently outperforms the greedy baseline, sometimes by as much as $50\%$. 

\begin{table}
\centering
\begin{tabular}{ |m{2.2em}||c|m{2.5em}|m{2.5em}|m{2.6em}|m{2.6em}|m{2.6em}| }
\hline 
$\sigma_m$ & $1$ & $2^{-1}$ & $2^{-2}$ & $2^{-3}$ & $2^{-4}$ & $2^{-5}$\\
\hline
\hline
ours & $-0.0002$ & $-0.003$ & $-0.003$ & $-0.013$ & $-0.025$ & $-0.043$\\
\hline 
greedy & $0.039$ & $0.116$ & $0.272$ & $0.417$ & $0.581$ & $0.624$\\
\hline
\end{tabular}
\caption{Suboptimality of our algorithm and the greedy algorithm with respect to the solution produced by the evolutionary algorithm.}
\label{table:noise-accuracy}
\end{table}

\subsection{Compute Speed}

Here we show the timings of the three procedures on a problem involving ten $\mathcal{C}$-agents and five $\mathcal{L}$-agents.
Table \ref{table:noise-accuracy} shows that our approach produces solutions that match the solution of the evolutionary algorithm to two significant digits at a fraction of the running time, as illustrated in Table \ref{compute_speed}. 

\begin{table}[h!]
\centering
\begin{tabular}{ |c||c|c|c| }
\hline 
$M$ & $2$ & $5$ & $10$\\
\hline
\hline
ours & $0.079$ & $0.108$ & $0.218$ \\
\hline 
greedy & $1.072$ & $2.605$ & $4.881$\\
\hline
global & $0.663$ & $4.970$ & $30.029$\\
\hline 
\end{tabular}
\caption{Computation time (in seconds) of our algorithm and the two alternate approaches. }
\label{compute_speed}
\end{table}

\section{Experiments}
\label{sec:experiments}

The setup for our simulation and real-world experiments involves a team of multiple UGVs and UAVs. To analyze the influence of the degree of structure of the environment and test the robustness of the system, we performed experiments across three types of environments, i.e., structured, semi-structured, and unstructured. We keep the same mission specification for real-world and simulated experiments in order to analyze the influence of real-world noise in controlled settings. Concretely, the UAV's waypoints are identical in real-world and simulation experiments. We identify the concept of waypoints and setpoints from \cref{sec:problem-form}. Instead of executing these waypoints once, the UAV is commanded to execute multiple rounds so that we can evaluate the robustness of our method for long-range missions.

The UGVs serve the role of $\mathcal{L}$-agents, and the UAV that of a $\mathcal{C}$-agent. The UGVs move along the ground plane at all times. 
To ensure the robustness of the state estimation module for our team of agents, we add a further constraint stipulating that at each setpoint, there must exist one yaw angle from which the UAV can view \textit{all} UGVs. Introducing additional variables, the unit heading vectors $\mathbf{n}_1, \mathbf{n}_2, ..., \mathbf{n}_N \in S^1$ along such yaw angles, the constraint-augmented problem becomes:

\begin{equation}
\begin{aligned}
\min_{\mathbf{z}_{1:M} \in (\mathbb{R}^2)^M, \ \mathbf{n}_{1:N} \in (S^1)^N} \ & \max_{i \in [N]} \ tr( \ \Sigma( \mathbf{x}_i; \ \mathbf{z}_{1:M} ) \ ) \\
s.t. & \\
& || \mathbf{z}_j ||_2 \leq R_{max} \hspace{9mm} \forall j \in [M] \\
& || \mathbf{z}_j - \mathbf{x}_i ||_2 \geq R_{min} \quad \forall i \in [N], \ \forall j \in [M] \\
& \angle ( \mathbf{n}_i, \ \mathbf{z}_j - \mathbf{x}_i ) \leq \frac{\alpha}{2} \quad \forall i \in [N], \ \forall j \in [M]. 
\end{aligned}
\end{equation}

The solution proceeds similarly to before. 
The addition of the aforementioned set of constraints translates into
\begin{equation}
\mathbf{n}_i \cdot (\mathbf{z}_j - \mathbf{x}_i) \geq \cos \left( \frac{\alpha}{2} \right) \sqrt{ || \mathbf{z}_j - \mathbf{x}_i ||2^2 + \delta^2 },
\end{equation}
which is a smooth approximation that becomes ideal in the limit as $\delta \downarrow 0$. 
Larger $\delta$ provides more conservative approximations of the field of view.

The rest of this section is organized as follows. The first subsection provides an overview of the software-hardware system used for our experiments. The second subsection presents simulation experiments, followed by real-world experiments.

\subsection{System overview}

\begin{figure}[t!]
\centering
\includegraphics[trim=0 0 0 0, clip,width=1.0\columnwidth]{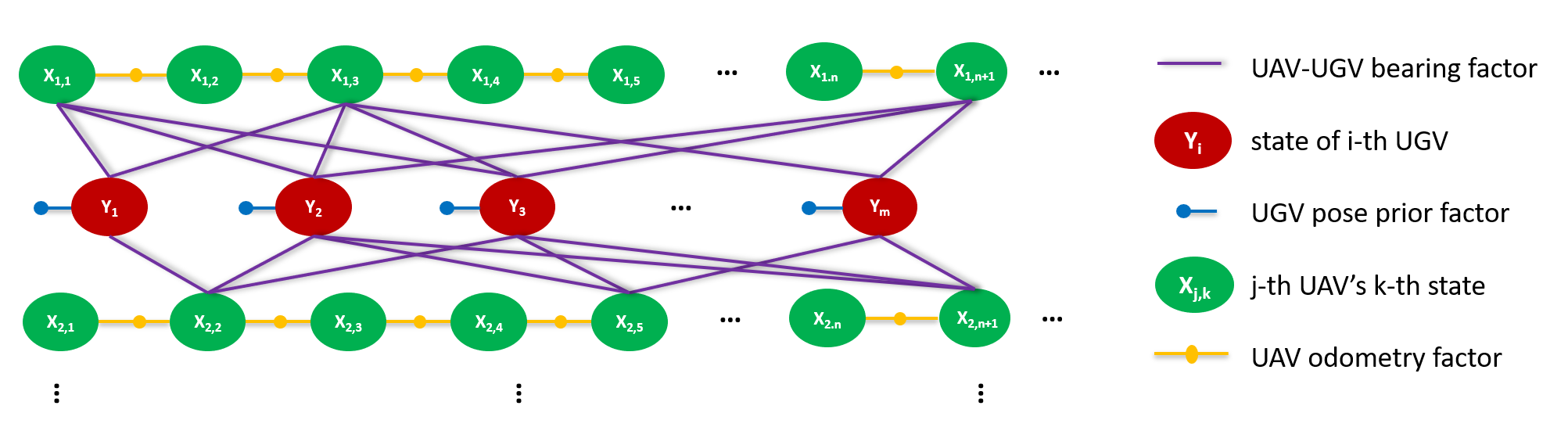}
    \caption{\textbf{Active collaborative team localization factor graph representation.} In our experiments, since the UGVs are static, each UGV is assigned exactly one state in the factor graph.}
    \label{fig:atl-factor-graph}
\end{figure}

\textbf{System overview (hardware):} In real-world experiments, our system consists of multiple UGVs and a UAV as shown in \cref{fig:robots}. The UGVs are equipped with ouster OS1-64 lidars. The UAV platform is equipped with a hardware-synchronized collection of greyscale stereo cameras, an IMU, and an RGB monocular camera. The greyscale stereo cameras and the IMU are used for stereo VIO \cite{sun2018robust}. In addition, the UAV platform also has an ouster OS1-64 lidar, but it is only used for obstacle avoidance; the UAV's localization relies only on cameras and the IMU. Our UGV and UAV platforms are capable of autonomous navigation in cluttered and GPS-denied environments using only onboard sensing and computation. Note that only the monocular RGB camera is used for detecting UGVs and generating bearing measurements. 

We carry out two sets of simulation experiments. In the first set, we have 1 UAV with 2 UGVs and 1 UAV with 3 UGVs. This allows us to study how much performance gain we can obtain by adding more UGVs into the environment. In the second set, we have 10 UAVs with 10 UGVs, which allows us to analyze how well our method can generalize to larger robot teams. To simulate real-world perception, we add noise to the UAV's camera measurements and odometry, as well as introducing dynamic objects in the environment.

\begin{figure}[t!]
\centering
\begin{subfigure}{\columnwidth}
\includegraphics[trim={0 0 0 75}, clip, width=\columnwidth, height = 0.75in]{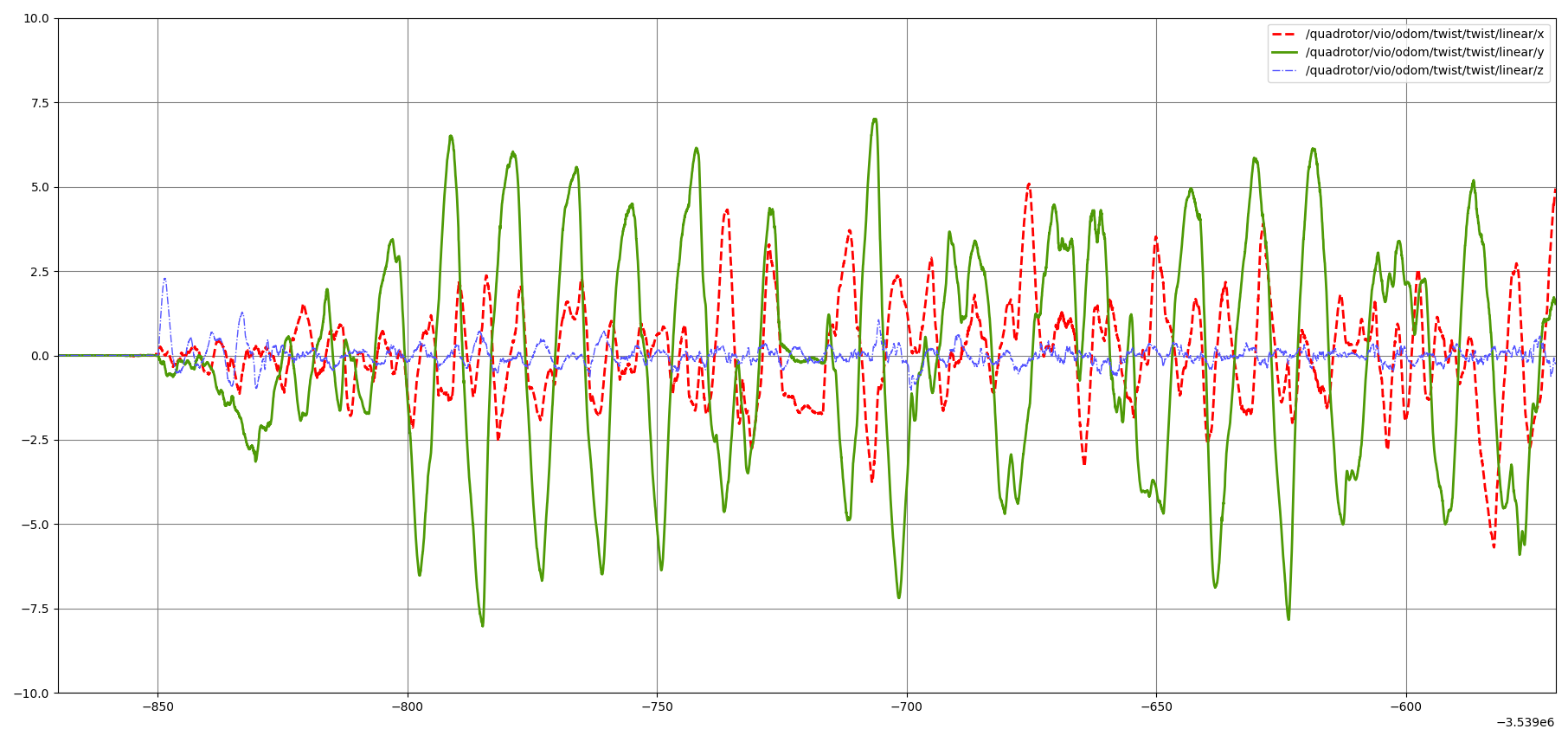}
\end{subfigure}

\begin{subfigure}{\columnwidth} 
\includegraphics[trim={0 0 0 75}, clip, width=\columnwidth, height = 0.75in]{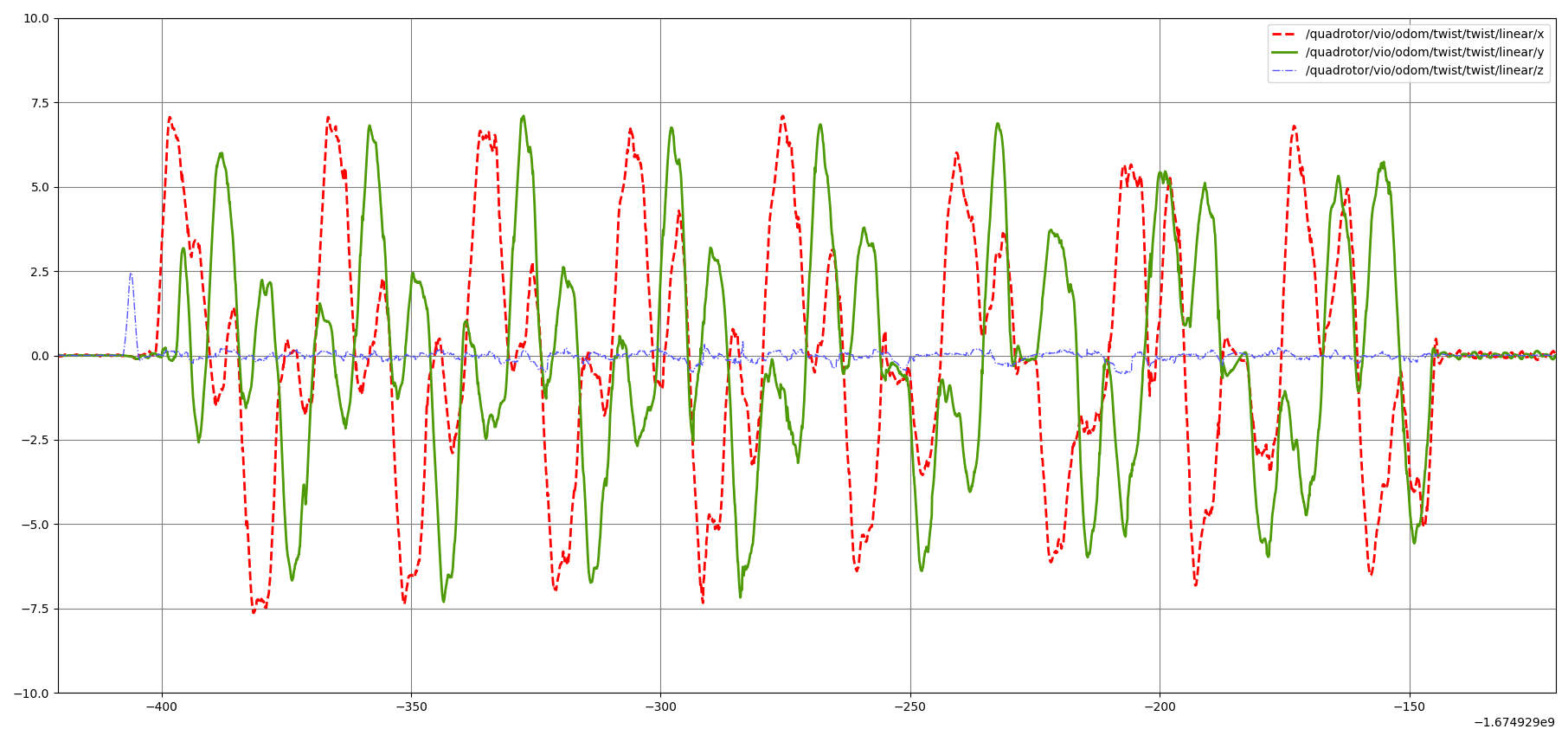}
\end{subfigure}
\caption{\textbf{Velocity profiles} for real-world experiments (top) and simulated experiments (bottom) illustrating aggressive maneuvers in experimentation and simulation. The curves show the velocities along X (dotted red), Y (solid green), and Z (dash-dotted blue) axes for flights in the unstructured environment. The vertical axis is the velocity, which ranges from -10 m/s to +10 m/s. The horizontal axis is the timestamp in seconds. The UAV accelerates drastically and frequently. Such aggressive motions pose significant challenges to state estimation.}
\label{fig:velocity-profile}
\end{figure}

\begin{figure*}[h!]
\begin{center} 
 \setlength
 \resizebox{1.0\textwidth}{!}
 {\begin{tabular}{||c | c | l | l | l ||} 
 \hline
 Environment / Traj. Length & Drift Red. & Ours Drift (X/Y/Z) (m) & Rand. UGVs Drift (X/Y/Z) (m)  & VIO Drift (X/Y/Z) (m) \\ 
  \hline \hline
Structured (2 UGVs) / 1.4 km & \textbf{86.29\%} &  \textbf{2.92} ( -1.21  /  -1.61  /  +2.11 ) & 4.16 ( -3.70  /  -0.33  /  +1.87 ) & 21.27 ( -12.74  /  -8.85  /  +14.56 ) \\
\hline
Structured (3 UGVs) / 1.4 km & \textbf{96.19\%} &  \textbf{0.81} ( -0.77  /  +0.20  /  +0.16 ) & 2.84 ( -1.36  /  -2.47  /  +0.30 ) &  21.27 ( -12.74  /  -8.85  /  +14.56 ) \\ 
  \hline  \hline
 Semi-Structured (2 UGVs) / 1.0 km &  \textbf{86.40\%} &  \textbf{3.21} ( -0.01  /  -2.73  /  -1.68 ) & 15.00 ( -9.38  /  +11.69  /  +0.70 ) &  23.57 ( -21.85  /  -6.29  /  +6.22 )  \\ 
 \hline
 Semi-Structured (3 UGVs) / 1.0 km & \textbf{88.38\%} &  \textbf{2.74} ( +0.15  /  -1.69  /  -2.15 ) & 8.01 ( -6.20  /  +5.04  /  +0.59 ) &  23.57 ( -21.85  /  -6.29  /  +6.22 ) \\ 
  \hline \hline
Unstructured (2 UGVs) / 1.4 km & \textbf{60.64\%} &  \textbf{7.50}  ( -2.77  /  +6.95  /  -0.51 ) & 16.04 ( -2.45  /  +15.85  /  -0.37 ) &  19.05 ( -1.15  /  +19.01  /  +0.56 ) \\ 
\hline
Unstructured (3 UGVs) / 1.4 km & \textbf{74.40\%} &  \textbf{4.88} ( -1.77  /  +4.10  /  -1.96 ) & 11.53 ( -2.91  /  +11.12  /  -0.94 ) &  19.05 ( -1.15  /  +19.01  /  +0.56 ) \\ 
  \hline \hline
\end{tabular}}
\end{center}
        \vspace{-0.08in}
\caption{{Quantitative results for Unity-based photo-realistic simulation experiments.} The performance improvement achieved by detecting and using UGVs as bearing constraints (3rd column) over using VIO alone (5th column) is shown in the 2nd column. All flight missions are autonomous with aggressive motions and a maximum design velocity of 13 m/s, as shown in the example in \cref{fig:velocity-profile}. The average drift reduction is 77.74\% with 2 UGVs and 86.32\% with 3 UGVs across all environments. Random positioning of UGVs produces much worse results (4th column) than our method. Therefore, the proposed method is superior to using VIO alone as well as randomly positioning UGVs.} 
\label{fig:quant-result-table-sim-3-robot}
\end{figure*}

\begin{figure*}[t!]
\begin{subfigure}{0.73\textwidth}
\begin{center} 
 \setlength
 \resizebox{1.0\textwidth}{!}
 {\begin{tabular}{||c | c | c| c ||} 
 \hline
UAV Index / Traj. Length  & Drift Red. & Ours Drift (X/Y/Z) (m)  & VIO Drift (X/Y/Z) (m)  \\ 
  \hline \hline
1st / 0.3087 km  & \textbf{75.35\%} &  {1.31 ( +1.00  /  -0.83  /  +0.12 )} &  {5.29 ( +0.46  /  -5.25  /  -0.51 ))}\\
\hline
2nd / 0.3835 km & \textbf{80.88\%} &  {1.80 ( -1.57  /  +0.74  /  -0.46 )}  &  {9.39 ( -8.44  /  +3.99  /  -1.00 )}\\ 
  \hline  
 3rd / 0.9042 km &  \textbf{62.68\%} &  {4.91 ( +0.50  /  -4.88  /  -0.19 )} &  13.15 ( -2.40  /  -12.33  /  -3.90 )  \\ 
 \hline
 4th / 0.1232 km & \textbf{57.73\%} &  {0.82 ( -0.30  /  -0.01  /  +0.76 )}  &  {7.30 ( +1.47  /  -7.13  /  +0.51)} \\ 
  \hline
5th / 0.3353 km & \textbf{60.04\%} &  {2.77 ( -2.67  /  -0.47  /  -0.58 )}  &  {6.94 ( -5.67  /  -3.83  /  -1.15 )}\\ 
\hline
6th / 0.1947 km & \textbf{68.22\%} &  {2.98 ( +1.05  /  -2.78  /  -0.18 )} & {9.37 ( -5.17  /  -7.73  /  -1.12 )} \\ 
\hline
7th / 0.5200 km & \textbf{81.69\%} &  {1.60 ( -1.12  /  +1.11  /  -0.26 )} &  {8.73 ( -6.21  /  +6.08  /  +0.80 )} \\ 
\hline
8th / 0.3310 km & \textbf{62.98\%} &  {2.73 ( -1.24  /  +2.39  /  +0.43 )} &  {7.37 ( -6.04  /  +4.15  /  +0.74 )}\\ 
\hline
9th / 0.5872 km  & \textbf{79.10\%} &  {2.59 ( -1.69  /  -1.82  /  -0.73 )} &  {12.39 ( +1.17  /  -12.07  /  +2.52 )} \\ 
\hline
10th / 0.4848 km  & \textbf{86.61\%} &  {1.28 ( -0.60  /  -1.13  /  +0.05 )} &  {9.56 ( +3.54  /  -8.81  /  -1.13 )} \\  
\hline
\end{tabular}}

\vspace{0.1in}

 \resizebox{1.0\textwidth}{!}
 {\begin{tabular}{||c | c | c| c | c | c ||} 
 \hline
 UAVs & Drift Red. Mean  & Drift Red. Std. Dev. & Traj. Length Median & Traj. Length Sum \\ 
  \hline \hline
1st-10th &  71.50\% &  9.85\% & 0.3594 km & 4.1726 km\\
\hline
\end{tabular}}
\end{center}

\end{subfigure}
\hspace{0.1in}
\begin{subfigure}{0.235\textwidth}
\centering
\includegraphics[trim=10 23 0 0, clip,width=1.0\columnwidth]{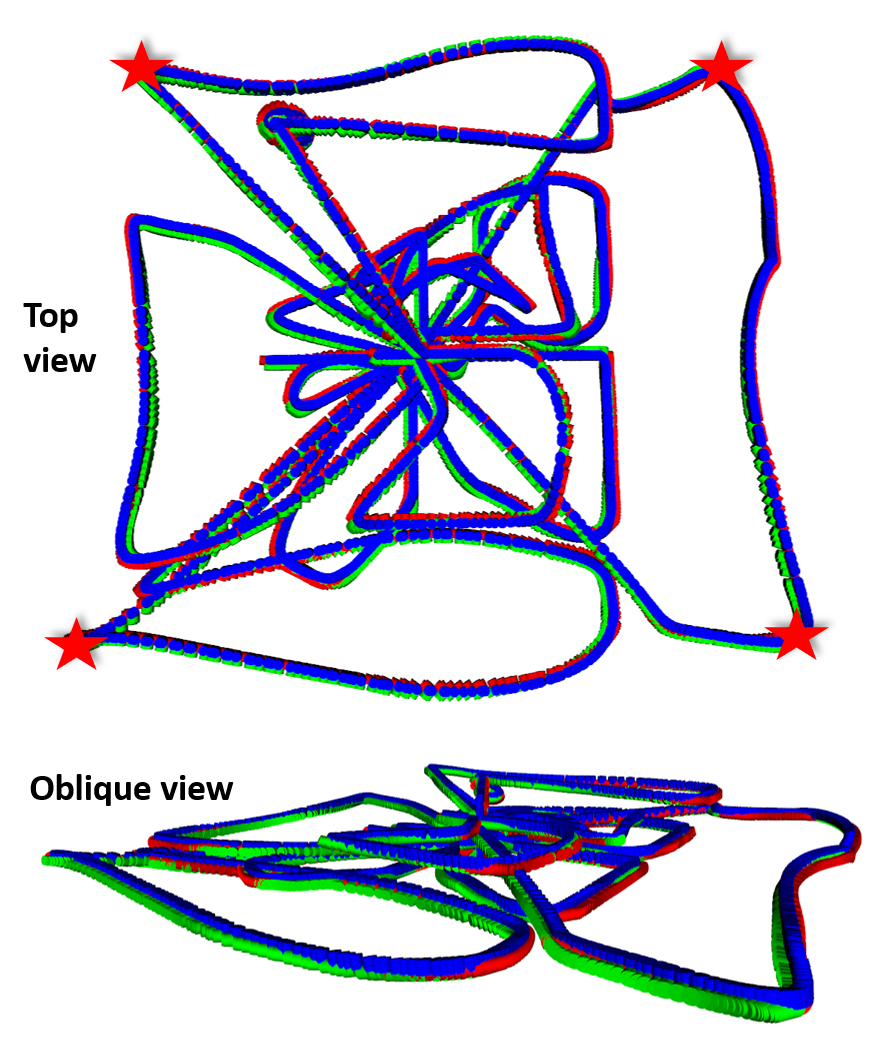}
\end{subfigure}
\caption{\textbf{Left: Quantitative results for simulation experiments with 10 UAVs and 10 UGVs in the unstructured environment.} The average drift correction across all UAVs is 71.50\%. The median trajectory length for each UAV is 359.4 m, which is significantly shorter than trajectories in \cref{fig:quant-result-table-sim-3-robot}. This leads to a smaller VIO drift and thus less room for drift reduction. The sum of all UAVs' trajectories is 4.2 km. \textbf{Right: Trajectories for all 10 UAVs.} Each UAV operates at a slightly different altitude and in a different part of the environment. Red stars show the corners of the environment.}
\label{fig:quant-result-table-sim-10-robot}
\end{figure*}


\textbf{System overview (estimation pipeline):} Our estimation pipeline builds upon a factor-graph-based method \cite{kaess2012isam2, dellaert2012factorgtsam}, which takes in VIO estimates and bearing measurements and outputs the UAVs' SE3 pose estimates. The choice of a factor graph approach over an EKF for the experiments is due to its better accuracy and consistency.  
The structure of our factor graph is shown in \cref{fig:atl-factor-graph}, which consists of two kinds of nodes (UAV and UGV poses) and three kinds of factors (UAV odometry factors, UAV-UGV bearing factors, and UGV pose prior factors). For the UAV odometry factors, we use the stereo-MSCKF algorithm \cite{sun2018robust} to estimate the relative transformation between two consecutive key poses, i.e., $\mathbf{p}^{vio}_{t} \ominus \mathbf{p}^{vio}_{t-1}$.

The generation of bearing factors is different between real-world and simulation experiments. For simulation experiments, the bearing measurements are generated by first calculating the relative bearing based on the ground-truth UGV and UAV poses and then adding Gaussian noise. For real-world experiments, the bearing factors are based on UGV detections. Since object detection is not the focus of our work, we put colored flags on the UGVs and detect them by image processing techniques including color filtering, opening, connected component analysis, and bounding box fitting. Data association is then carried out by projecting the 3D positions of the UGV centers back onto the image plane. The camera pose used during this projection process is $\hat{\mathbf{p}}_t = \mathbf{p}^g_{t-1} \oplus (\mathbf{p}^{vio}_{t} \ominus \mathbf{p}^{vio}_{t-1})$, i.e., the composition of the latest key pose from the factor graph ($\mathbf{p}^g_{t-1}$) and the relative motion estimated by the VIO. 

The UGV pose prior factor is generated by integrating poses estimated by a state-of-the-art lidar-inertial odometry algorithm \cite{bai2022faster}. Since the UGVs' trajectories are much shorter in distance and less aggressive in motion, the lidar-inertial odometry can estimate their poses with only centimeter-level drift, which is several orders of magnitude smaller than the UAV's VIO drift. We use the GTSAM library \cite{dellaert2012factorgtsam, kaess2012isam2} as the backend.

\subsection{Simulation Experiments}
\begin{figure}[h!]
\centering
\includegraphics[trim=0 0 0 0, clip,width=1.0\columnwidth]{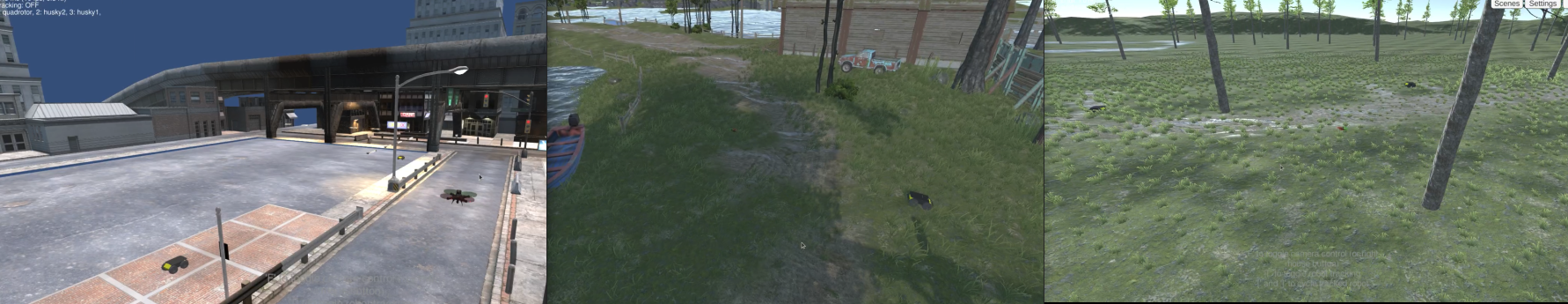}
    \caption{\textbf{Simulation experiment environments:} Structured (left), semi-structured (middle), unstructured (right).}
    \label{fig:simulation-environments}
\end{figure}

\begin{figure*}[t!]
\begin{center} 
 \setlength
 \resizebox{0.8\textwidth}{!}
 {\begin{tabular}{||c | c | l | l ||}
 \hline
 Environment / Traj. Length & Drift Reduction & ~~~ Ours Drift (X/Y/Z) (m) & VIO Drift (X/Y/Z) (m) \\ 
  \hline \hline
Structured / 2 km & \textbf{42.30\%} & 4.49 ( +4.22  /  +0.62  /  -1.41 ) &  7.79 ( +7.02  /  +2.29  /  -2.47 )\\ 
  \hline  \hline
 Semi-Structured / 1 km & \textbf{68.28\%} &  1.85 ( +0.46  /  +0.27  /  -1.77 ) &  5.83 ( +3.10  /  +4.26  /  -2.49 )  \\
  \hline 
Semi-Structured Aggressive / 2.5 km & \textbf{70.89\%} & 3.66 ( -1.58  /  +3.17  /  -0.92 ) &  12.57 ( +2.49  /  -7.98  /  -9.39 ) \\ 
  \hline \hline
Unstructured / 2.5 km & \textbf{79.08\%} &  0.70 ( +0.69  /  +0.09  /  +0.06 ) &  3.34 ( -0.78  /  -2.90  /  -1.46 ) \\ 
  \hline
Unstructured Aggressive / 3 km & \textbf{83.37\%} &  2.80 ( -0.80  /  +2.68  /  -0.05 ) &  16.82 ( -13.22  /  -10.08  /  -2.58 )  \\ 
  \hline
\end{tabular}}
\end{center}
        \vspace{-0.08in}
\caption{\textbf{Quantitative results for real-world experiments.} The performance improvement by our method is shown in the 2nd column. The average drift reduction is 68.78\% with 2 UGVs. Going from structured to unstructured environments, UGV detection becomes increasingly reliable. This is due to numerous false positives in the structured environment caused by objects with similar appearance to the UGV, such as cars, buses, and fire hydrants. In contrast, there are much fewer false positives in the unstructured environment. Furthermore, in unstructured environments, the VIO becomes less reliable due to the lack of static and reliable geometric features. As a result, the performance gain from the UGV bearing measurements is much more significant in unstructured environments. The performance gain in the real world with 2 UGVs is $\sim$10\% lower than in the simulation. One major difference we observe is that, in simulation, although we add noise to the bearing measurements, the UGV detection and data association are noise-free. This also shows the importance of accurate and robust UGV detection and data association. Improving these will be the future work of this paper.}
\label{fig:quant-result-table}
\end{figure*}

\begin{figure}[h!]
\centering
\includegraphics[trim=0 0 0 0, clip,width=0.75\columnwidth]{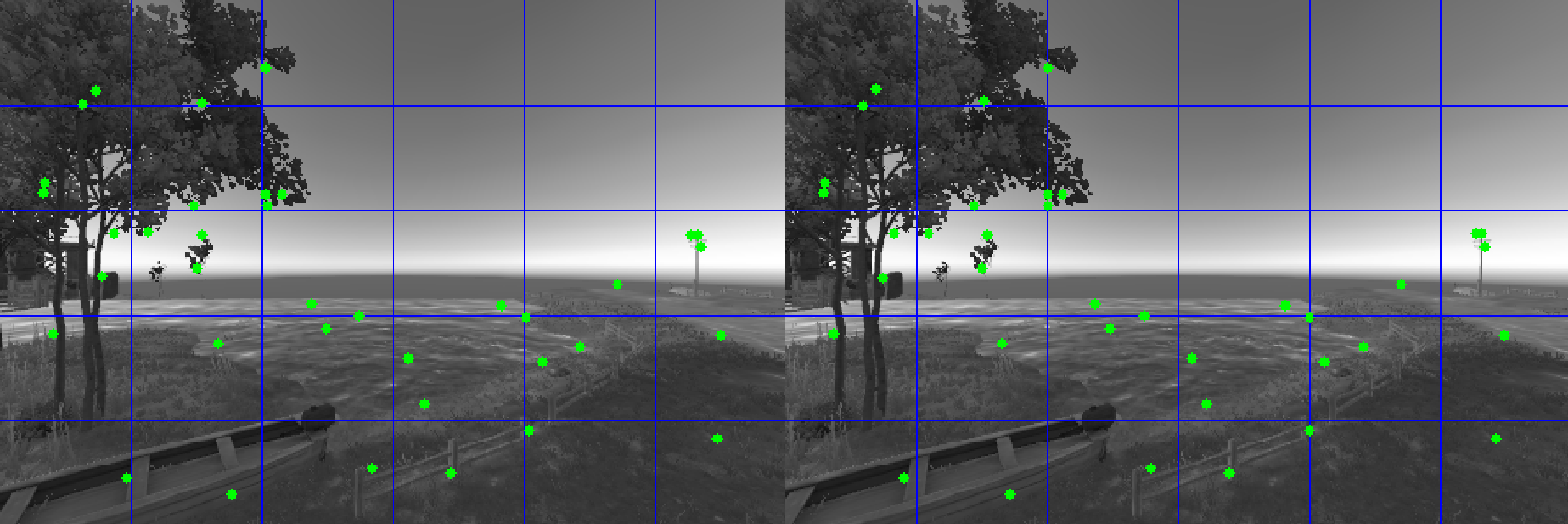}
    \caption{\textbf{Simulation stereo images and feature extraction example.} The simulated images are photo-realistic. The plot is generated by the stereo-MSCKF algorithm~\cite{sun2018robust}, where features extracted and tracked by the VIO are illustrated by green dots. Some features are extracted from dynamic objects (e.g., tree leaves, grass, and water), leading to additional drifts in VIO.}
    \label{fig:example-stereo-vio-features}
\end{figure}

\textbf{Experiment environments:} We use a custom Unity-based simulator. The simulator is integrated with ROS and can simulate photorealistic sensor data such as RGB and depth images, and lidar point clouds. We choose the simulation environments that approximate the real-world experiment environments, as shown in \cref{fig:simulation-environments}. The structured environment is a city-like environment. The semi-structured environment is a rural area with flooded grounds and some structures like farmhouses, bridges, and cars. The unstructured environment is a sparse forest. These simulation environments have dynamic objects such as grass, water, and leaves.  

\textbf{Quantitative results and analysis:} The results of simulation experiments are shown in \cref{fig:quant-result-table-sim-3-robot}. We use our autonomous flight software stack \cite{liu2022large} in the simulator to accomplish all these flight missions. The trajectory planner is set to have a maximum acceleration of 9.90 $m/s^2$ ($\sim$1g) and a maximum velocity of 13 $m/s$. Such aggressive motions and the simulated noisy visual perception lead to significant VIO drifts, as shown in the last column of \cref{fig:quant-result-table-sim-3-robot}. An illustration showing the aggressive motions of one of the simulated experiments is in \cref{fig:velocity-profile}.

The second column of \cref{fig:quant-result-table-sim-3-robot} shows that by using our method, the VIO drifts are significantly reduced. The drift reduction is consistent across all environments. With 2 UGVs, the drift is reduced by 77.74\%. With 3 UGVs, the position drift is reduced by 86.32\%.

Such drastic drift reduction results from optimizing the positions of UGVs using our method, and cannot be obtained by merely introducing them into the environment at arbitrary locations. To show this, we randomly position the UGVs within the region enclosed by the UAV's waypoints. The UAV is commanded to execute the exact same missions, and uses the same factor graph estimation pipeline. The results are shown in the third column of \cref{fig:quant-result-table-sim-3-robot}. Our method consistently outperforms this random positioning method by a significant margin, resulting in 59.43\% less drift on average across all experiments.

With larger UAV and UGV teams, the results are shown in \cref{fig:quant-result-table-sim-10-robot}. Our algorithm can position the UGVs in a configuration so that the drift of each of the 10 UAVs is drastically reduced. The average drift reduction is 71.50\%. Compared to results in \cref{fig:quant-result-table-sim-3-robot}, the drift reduction is slightly smaller. This is a result of two factors: First, the average trajectory length of experiments in \cref{fig:quant-result-table-sim-3-robot} is 304\% the average trajectory length of experiments in \cref{fig:quant-result-table-sim-10-robot}. Since the VIO estimates the relative motion w.r.t. the previous time step, the errors in pose estimates accumulate over time. On the contrary, our method utilizes measurements from UGVs which can provide global correction to the odometry drift. Therefore, the longer the flight trajectory, the bigger drift reduction is expected. Second, even though we have more UGVs in \cref{fig:quant-result-table-sim-10-robot}, the UAV to UGV ratio is smaller. This means that on average fewer UGV resources are allocated for each UAV.

This set of simulation experiments demonstrates that our proposed method can consistently help the UAVs minimize their odometry drift, across all environments and flight missions. The corrected drift is an order of magnitude smaller than the raw VIO. Such drastic drift correction is critical when UAVs execute autonomous flight missions at scale.

\subsection{Real World Experiments}

\textbf{Experiment environments:} To quantitatively verify the performance of our system, we perform experiments across multiple real-world scenarios, including an unstructured grass-covered area, an open structured area, and a moderately cluttered parking lot. To ensure that the UAV has the desired maneuvers (e.g., aggressive motions, landing at exactly the same position as takeoff), we manually piloted the UAV to fly through the waypoints. An illustration of these environments is shown in \cref{fig:exp-environments}.

\begin{figure}[h!]
\centering
\includegraphics[trim=0 0 0 190, clip,width=1.0\columnwidth]{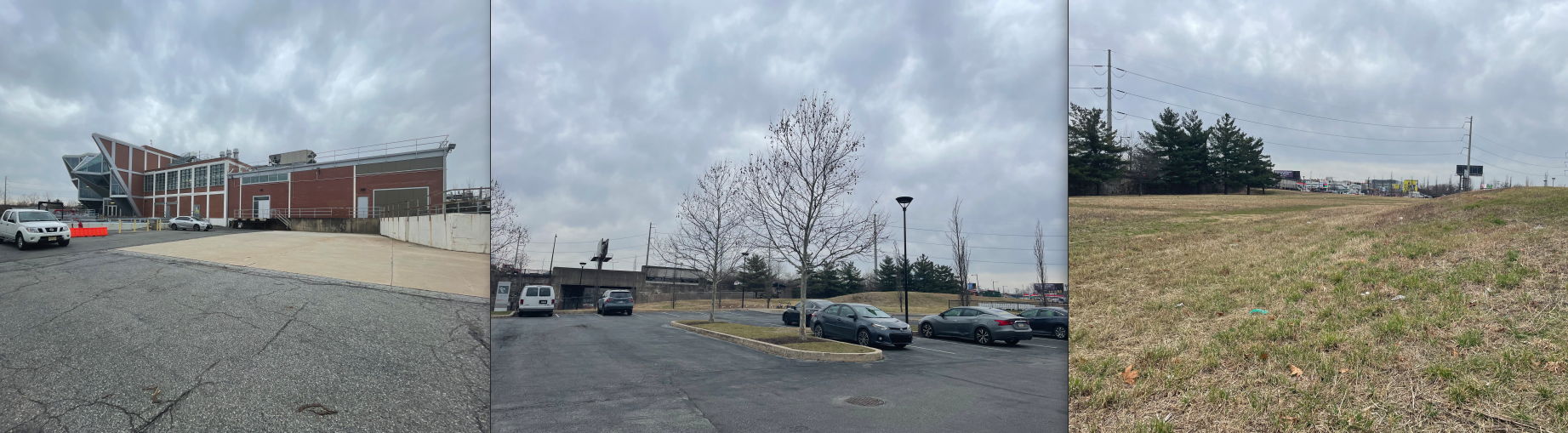}
    \caption{\textbf{Real world experiment environments:} Structured (left), semi-structured (middle), unstructured (right).}
    \label{fig:exp-environments}
\end{figure}

\textbf{Quantitative results and analysis:} The real world experiment results are shown in \cref{fig:quant-result-table}. 
The drift reduction in the unstructured environment is most significant, with an average reduction of 81.23\%. A counter-intuitive result is that the resulting drift using our method decreases with the increase of the unstructuredness of the environment. This is mainly due to false data association since the structured environment has objects with similar colors as our UGV marker flag colors (buses, fire hydrants, construction barrels, etc.). On the contrary, the unstructured environment provides good contrast in color since the majority of the environment is covered by green grass, as shown in \cref{fig:exp-environments}. This can be addressed by using a better object detector and a more robust data association strategy.

To sum up, from multiple long-range UAV flight experiments in environments with a large variance in appearance, clutteredness, and scale, we show by using UGV detections as bearing constraints, the UAV's state estimation is significantly improved. Such improvement is consistent across all scenarios, but most significant when the UGVs are observed more frequently, and data association is more reliable. This indicates that the theoretical and algorithmic aspects of our system are correct and effective, and the system is robust to real-world noise.

\section{Discussion}
Throughout the experiments, we have demonstrated the proposed method's robustness and performance across various environments. Here, we will present its potential applications.

\textit{Coverage problem}: The missions demonstrated in our experiments all had prescribed setpoints for UAVs. However, since our algorithm's computational complexity is not dominated by the number of UAVs, it can be used in the coverage problem where the robot team must cover and map a given region. The user can specify the environment, and a sampling-based method can be used to determine the positions of UAVs' setpoints \cref{fig:coverage-application-example}. The UGVs will be positioned by the proposed method based on these setpoints to improve the UAVs' localization accuracy.

\begin{figure}[h!]
\centering
\includegraphics[trim=0 0 0 0, clip,width=1.0\columnwidth]{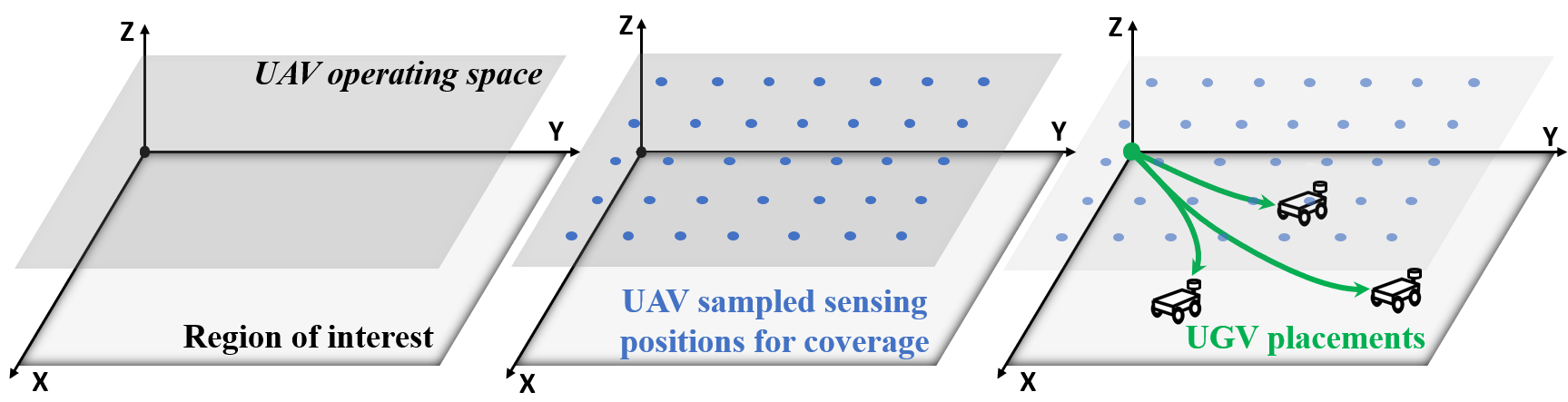}
    \caption{Coverage application example.}
    \label{fig:coverage-application-example}
\end{figure}

\textit{Online active collaborative localization}: In our current experiments, we assumed that there was no communication during the mission execution. However, when the robots are allowed to communicate, our method can enable the UGVs to adapt their positions dynamically based on the performance of the collaborative localization algorithm using feedback of the errors in the trajectories of UAVs.
Specifically, when this happens, the UAVs can inform the UGVs to actively relocate themselves using the proposed method to minimize the UAVs' localization uncertainties.

\textit{Robust perception in perturbed environments}: Another important application of our method is to the problems where the environments are susceptible to changes. For example, some landmarks may be destroyed or disappear as time passes (e.g. rocks disappear with the tide, and buildings might get destroyed by earthquakes). It is challenging for UAVs to predict which part of the environment is perturbed. Our method offers UAVs a resort when they are uncertain about their localization.

\section{Conclusion}

This paper proposes an algorithm for active collaborative localization of a large number of SWAP-constrained robots by optimally positioning (and potentially navigating) a small number of more capable teammates. 
Specifically, we consider a team of aerial robots with cameras and IMUs that can leverage measurements of ground vehicles equipped with heavier sensors and processors. 
Numerical results show that our smooth optimization approach for the underlying active robot positioning problem outperforms the greedy algorithm in terms of accuracy, and matches the objective value of a computationally intensive heuristic evolutionary algorithm for global optimization while running in real time. 
Through experiments in photorealistic simulation environments, the proposed method reduces the UAV odometry drift by 90\% with 3 UGVs. 
It also outperforms randomly positioning UGVs, and shows a robust and significant reduction in drift in real-world experiments. 
The performance of our system is consistent throughout a range of scenarios that feature large variations in robot motions as well as the amount of structure in the environment. 
This robust and drastic improvement in state estimation enables long-range autonomous navigation. 
The proposed active collaborative localization method can be used in various real-world applications in which we encounter environments without texture or features for aerial flight. 
Future work will address the use of errors in localization as a feedback mechanism to modify trajectories for the $\mathcal{C}$-agents and allow bidirectional flow of information between $\mathcal{C}$-agents and $\mathcal{L}$-agents. 
We also acknowledge the need to improve the robustness of our method to potential errors in data association and scale up to a bigger team of robots.

\bibliographystyle{IEEEtran}
\bibliography{papers}
\section{Appendix}
\label{sec:appendix_a}

We split the proof of Theorem \ref{theorem:theorem_multiplicative} into two claims. 
Roughly speaking, our reasoning hinges on the fact that the trace of the covariance matrix is just the sum of inverses of the eigenvalues of the information matrix. 
The rest of the proof involves analyzing how the sum of inverses of a sequence of positive real numbers behaves under suitable perturbations. 
This is precisely the content of the following 

\begin{claim} \label{cl:claim_one}
Let $\delta, \eta > 0$ be a fixed pair of real numbers, and define
\begin{equation} \label{eq:max_s0}
s_0 = \frac{\eta}{\eta + 1} \delta^{-1}.
\end{equation}
For any pair of sequences $(\lambda_k)_{k = 1}^d, \ (\tilde{\lambda}_k)_{k = 1}^d \in (0, \infty)^d$ such that $|\tilde{\lambda}_k - \lambda_k | \leq \delta \ \forall 1 \leq k \leq d$, we have the implication
\begin{equation} \label{eq:claim_one_result}
\sum_{k = 1}^{d} \frac{1}{\lambda_k} = s \leq s_0 
\quad \Rightarrow \quad 
\frac{s}{1 + \eta} \leq \sum_{k = 1}^{d} \frac{1}{\tilde{\lambda}_k} \leq s (1 + \eta).
\end{equation}
\label{eqn:claim-1-implications}
\end{claim}

\begin{proof}
To prove the inequality on the right, it suffices to show 
\begin{equation} \label{eq:claim_one_pf_right}
\sum_{k = 1}^{d} \frac{1}{\lambda_k} = s \leq s_0 
\quad \Rightarrow \quad
\sum_{k = 1}^{d} \frac{1}{\lambda_k - \delta} \leq s (1 + \eta).
\end{equation}
Indeed, since $\lambda_k > 0 \ \forall k$, we have 
\begin{equation}
\frac{1}{\lambda_k} \leq s \ \Rightarrow \ \lambda_k \geq s^{-1} \geq s_0^{-1} = \frac{1 + \eta}{\eta} \delta > \delta \ \forall k,
\end{equation}
showing that the expression on the right of Equation \ref{eq:claim_one_pf_right} is well-defined. 
Now, consider the following optimization problem
\begin{equation}
\begin{aligned}
\sup_{(\lambda_k)_{k = 0}^d \in (0, \infty)^d} & \ \sum_{k = 1}^{d} \frac{1}{\lambda_k - \delta} \\
& \ s.t. \ \sum_{k = 1}^{d} \frac{1}{\lambda_k} = s.
\end{aligned}
\end{equation}
To show that it is bounded above by $(1 + \eta) s$, upon introducing variables $a_k := \lambda_k^{-1}$, we reformulate it as follows 
\begin{equation}
\begin{aligned}
\max_{(a_k)_{k = 0}^d \in [0, s]^d} & \ \sum_{k = 1}^{d} \frac{1}{a_k^{-1} - \delta} \\
& \ s.t. \ \sum_{k = 1}^{d} a_k = s.
\end{aligned}
\label{eqn:aux-right}
\end{equation}
We note that the function $f(x) = \frac{1}{x^{-1} - \delta}$ is convex on $x \in [0, s] \subseteq [0, \delta^{-1}]$ due to
$f''(x) = \frac{2\delta}{(1 - x \delta)^3} > 0$ on the the mentioned interval. 
As a result, Problem \ref{eqn:aux-right} involves maximizing a convex function subject to a simplex constraint. 
Hence, an optimum is attained at one of the extreme vertices of the feasible region, which all take the form $(s, 0, \dots, 0)$ up to a permutation of indices. 
In any case, all such extreme points have the objective value 
\begin{equation}
\frac{1}{s^{-1} - \delta} = \frac{s}{1 - s \delta} \leq \frac{s}{1 - s_0 \delta}  = (1 + \eta) s,
\end{equation}
thus proving the desired inequality. 

Similarly, to prove the inequality on the left, it suffices to show 
\begin{equation}
\sum_{k = 1}^{d} \frac{1}{\lambda_k} = s \leq s_0 
\quad \Rightarrow \quad
\sum_{k = 1}^{d} \frac{1}{\lambda_k + \delta} \geq \frac{s}{1 + \eta}.
\end{equation}
Now, consider the following optimization problem
\begin{equation}
\begin{aligned}
\inf_{(\lambda_k)_{k = 0}^d \in (0, \infty)^d} & \ \sum_{k = 1}^{d} \frac{1}{\lambda_k + \delta} \\
& \ s.t. \ \sum_{k = 1}^{d} \frac{1}{\lambda_k} = s.
\end{aligned}
\end{equation}
To show that it is bounded below by $\frac{s}{1 + \eta}$, upon introducing variables $a_k := \lambda_k^{-1}$, we reformulate it as follows 
\begin{equation}
\begin{aligned}
\min_{(a_k)_{k = 0}^d \in [0, s]^d} & \ \sum_{k = 1}^{d} \frac{1}{a_k^{-1} + \delta} \\
& \ s.t. \ \sum_{k = 1}^{d} a_k = s.
\end{aligned}
\label{eqn:aux-left}
\end{equation}
We note that the function $f(x) = \frac{1}{x^{-1} + \delta}$ is concave on $x \in [0, s] \subseteq [0, \delta^{-1}]$ due to
$f''(x) = \frac{-2\delta}{(1 + x \delta)^3} < 0$. 
As a result, Problem \ref{eqn:aux-left} involves minimizing a concave function subject to a simplex constraint. 
Hence, an optimum is attained at one of the extreme vertices of the feasible region, which all take the form $(s, 0, \dots, 0)$ up to a permutation of indices.
In any case, all such extreme points have the objective value 
\begin{equation}
\frac{1}{s^{-1} + \delta} = \frac{s}{1 + s \delta} \geq \frac{s}{1 + s_0 \delta}  = \frac{1 + \eta}{1 + 2 \eta} s \geq \frac{s}{1 + \eta},
\end{equation}
thus proving the desired inequality.
\end{proof}

\begin{claim} \label{cl:claim_two}
Consider an arbitrary $\delta > 0$, $\mathbf{x} \in \mathbb{R}^3$, and suppose 
\begin{equation}
\frac{\delta}{||\mathbf{x} - \mathbf{z}_j||_2} \leq \zeta^{-1} \quad \forall 1 \leq j \leq M,
\end{equation} 
for some $\zeta > 0$. 
Then,   
\begin{equation}
|| \mathcal{J}_{\delta}(\mathbf{x} ; \mathbf{z}_{1:M}) - \mathcal{J}_{0}(\mathbf{x} ; \mathbf{z}_{1:M}) ||_2 \leq \frac{M \ \sigma_m^{-2}}{\underset{j \leq M}{\min} \ || \mathbf{x} - \mathbf{z}_j||_2^2} \frac{1}{(1 + \zeta^2)}.
\end{equation}
\end{claim}

\begin{proof}
Letting $\Delta_j = || \mathbf{z}_j - \mathbf{x} ||_2 \ \forall 1 \leq j \leq M$, we have 
\begin{equation}
\begin{aligned}
S^{\delta}(\mathbf{x}; \mathbf{z}_j) & := \frac{\partial \mathbf{h}_{\delta}^T}{\partial \mathbf{x}}(\mathbf{x}, \mathbf{z}_j) \frac{\partial \mathbf{h}_{\delta}}{\partial \mathbf{x}}(\mathbf{x}, \mathbf{z}_j) \\
& = \frac{1}{\Delta_j^2 + \delta^2}(I_3 - \frac{\Delta_j^2 + 2 \delta^2}{(\Delta_j^2 + \delta^2)^2} (\mathbf{z}_j - \mathbf{x})(\mathbf{z}_j - \mathbf{x})^T).
\end{aligned}
\end{equation}
By noting that the family of symmetric matrices $(S^{\delta}(\mathbf{x}; \mathbf{z}_j))_{\delta \in \mathbb{R}}$ may be jointly diagonalized with respect to an orthonormal basis containing $\frac{\mathbf{z}_j - \mathbf{x}}{\Delta_j}$, and noting that the spectral norm of a symmetric matrix is just the maximum absolute value of its eigenvalues, we get
\begin{equation}
|| S^{\delta}(\mathbf{x}; \mathbf{z}_j) - S^{0}(\mathbf{x}; \mathbf{z}_j) ||_2 \leq \Delta_j^{-2} \frac{1}{1 + \zeta^2}.
\end{equation}
By the triangle inequality, we have
\begin{equation}
\begin{aligned}
\frac{ || \mathcal{J}_{\delta}(\mathbf{x} ; \mathbf{z}_{1:M}) - \mathcal{J}_{0}(\mathbf{x} ; \mathbf{z}_{1:M}) ||_2 }{\sigma_m^{-2}}& = 
|| \sum_{j = 1}^{M} (S^{\delta}(\mathbf{x}; \mathbf{z}_j) - S^{0}(\mathbf{x}; \mathbf{z}_j)) ||_2 \\
& \leq  \sum_{j = 1}^{M} || (S^{\delta}(\mathbf{x}; \mathbf{z}_j) - S^{0}(\mathbf{x}; \mathbf{z}_j)) ||_2 \\
& \leq  \sum_{j = 1}^{M} \Delta_j^{-2} \frac{1}{1 + \zeta^2} \\
& \leq  \frac{M}{\underset{j \leq M}{\min} \ \Delta_j^{2}} \frac{1}{1 + \zeta^2},
\end{aligned}
\end{equation}
as desired.
\end{proof}

Now we leverage Claims \ref{cl:claim_one} and \ref{cl:claim_two} to provide a proof of Theorem \ref{theorem:theorem_multiplicative}.

\begin{proof}(of Theorem \ref{theorem:theorem_multiplicative})
Since the prior covariance is positive definite, $\mathcal{J}_{\delta} \succ 0$ for all $\delta$.
Next, denote the eigenvalues of $\mathcal{J}_{\delta}$ in decreasing order by $\lambda_1^{\delta} \geq \lambda_2^{\delta} \geq \lambda_3^{\delta} > 0$.
It then follows that 
\begin{equation}
tr( (\mathcal{J}_{\delta})^{-1} ) = \sum_{j=1}^{3} \frac{1}{\lambda_j^{\delta}}.
\end{equation}
By Weyl's inequality \cite{weyl1912asymptotische}, we have 
\begin{equation}
| \lambda_j^{\delta} - \lambda_j^{0} | \leq || \mathcal{J}_{\delta} - \mathcal{J}_{0} ||_2 \leq 
\frac{M \ \sigma_m^{-2}}{R_{min}^2} \frac{1}{(1 + \zeta^2)} \quad \forall j \leq 3,
\end{equation}
where the latter inequality follows from Claim \ref{cl:claim_two}. 
As a result, making substitutions 
\begin{equation}
\lambda_j \rightarrow \lambda_j^{\delta}, \ \tilde{\lambda}_j \rightarrow \tilde{\lambda}_j^{\delta}, \ 
\delta \rightarrow \frac{M \ \sigma_m^{-2}}{R_{min}^2} \frac{1}{(1 + \zeta^2)}
\end{equation}
in Claim \ref{cl:claim_one}, we get that for $s_0$ given by Equation \ref{eqn:allowed-s}, the theoreem holds.
\end{proof}

Finally we turn to proving Corollary \ref{cor:guarantee_corollary} with the help of an additional claim.

\begin{claim}
Denote the value of Problem \ref{prob:problem_delta} at a given $\mathbf{z}_{1:M}$ for a given $\delta$ by $V^{\delta}$.
Within the same setup as for Theorem \ref{theorem:theorem_multiplicative}, we have 
\begin{equation}
V^{\delta} = s \leq s_0 \ \Rightarrow \ \frac{V^{\delta}}{1 + \eta} \leq V^{0} \leq V^{\delta} (1 + \eta).
\end{equation}
By symmetry the same implication holds with the roles of $V^{\delta}$ and $V^{0}$ reversed.
\end{claim}

\begin{proof}
To prove the right-hand inequality, note that by the definition of $V^{\delta}$, we have $\forall i \leq N$
\begin{equation}
tr( \mathcal{J}_{\delta}(\mathbf{x}_i ; \mathbf{z}_{1:M})^{-1} ) \leq V_{\delta} = s \leq s_0.
\end{equation}
Hence, by Claim \ref{cl:claim_one}, we have $\forall i \leq N$
\begin{equation}
tr( \mathcal{J}_{0}(\mathbf{x}_i ; \mathbf{z}_{1:M})^{-1} ) \leq tr( \mathcal{J}_{\delta}(\mathbf{x}_i ; \mathbf{z}_{1:M})^{-1} )(1 + \eta) \leq (1 + \eta) V^{\delta},
\end{equation}
and by taking the maximum of the latter inequality over $i \leq N$, the inequality follows.
To prove the left-hand inequality, again by definition, there exists an $i \leq N$ such that 
\begin{equation}
tr( \mathcal{J}_{\delta}(\mathbf{x}_i ; \mathbf{z}_{1:M})^{-1} ) = V^{\delta} = s \leq s_0.
\end{equation}
Hence, by Claim \ref{cl:claim_one}, we have
\begin{equation}
V^{0} \geq tr( \mathcal{J}_{0}(\mathbf{x}_i ; \mathbf{z}_{1:M})^{-1} ) \geq \frac{tr( \mathcal{J}_{\delta}(\mathbf{x}_i ; \mathbf{z}_{1:M})^{-1} )}{1 + \eta} = \frac{V^{\delta}}{1 + \eta},
\end{equation}
as desired.
\end{proof}

\begin{proof} (of Corollary \ref{cor:guarantee_corollary})
Let $\mathbf{z}_{*}^{\delta}$ and $\mathbf{z}_{*}^{0}$ be optimal solutions for Problem \ref{prob:problem_delta} and Problem \ref{prob:problem_one}, respectively. 
By assumption, we have 
\begin{equation}
\mathbf{z}_{*}^{\delta} \in \mathcal{S} \ \Rightarrow \ V^{\delta}( \mathbf{z}_{*}^{\delta} ) \leq s_0.
\end{equation}
By Claim \ref{cl:claim_one}, we then have 
\begin{equation}
V^{0}( \mathbf{z}_{*}^{\delta} ) \leq (1 + \eta) V^{\delta}( \mathbf{z}_{*}^{\delta} ) \leq (1 + \eta) s_0.
\end{equation}
We now distinguish two cases. 
First, if $V^{0}(\mathbf{z}_{*}^{0}) > s_0$, the inequality trivially holds.
Otherwise,  $V^{0}(\mathbf{z}_{*}^{0}) \leq s_0$, and by the previous claim, we have 
\begin{equation}
V^{0}(\mathbf{z}_{*}^{0}) \geq \frac{V^{\delta}(\mathbf{z}_{*}^{0})}{1 + \eta} \geq \frac{V^{\delta}(\mathbf{z}_{*}^{\delta})}{1 + \eta} \geq \frac{V^{0}(\mathbf{z}_{*}^{\delta})}{(1 + \eta)^2},
\end{equation}
and so the result follows.
\end{proof}

\end{document}